\documentclass{article} %
\usepackage{iclr2026_conference,times}

\usepackage{amsmath,amsfonts,bm}

\def\eqref#1{equation~\ref{#1}}

\def\1{\bm{1}}

\def\eps{{\epsilon}}

\def\vb{{\bm{b}}}

\def\vp{{\bm{p}}}

\def\vw{{\bm{w}}}
\def\vx{{\bm{x}}}

\def\vz{{\bm{z}}}

\DeclareMathAlphabet{\mathsfit}{\encodingdefault}{\sfdefault}{m}{sl}
\SetMathAlphabet{\mathsfit}{bold}{\encodingdefault}{\sfdefault}{bx}{n}

\newcommand{\R}{\mathbb{R}}

\usepackage[utf8]{inputenc} %
\usepackage[T1]{fontenc}    %
\usepackage{hyperref}       %
\usepackage{url}            %
\usepackage{booktabs}       %
\usepackage{amsfonts}       %
\usepackage{nicefrac}       %
\usepackage{microtype}      %
\usepackage{xcolor}         %

\usepackage{graphicx}
\usepackage{booktabs} %

\usepackage{soul}
\usepackage[normalem]{ulem}

\usepackage{algorithm}
\usepackage[noend]{algorithmic}

\usepackage{amsmath}
\usepackage{amssymb}
\usepackage{mathtools}
\usepackage{amsthm}

\usepackage{dirtytalk}
\usepackage{multirow}
\usepackage{multicol}
\usepackage{lipsum}  
\usepackage{wrapfig}
\usepackage{subfig}
\usepackage{xspace}

\newcommand{\red}[1]{#1}

\newtheorem{theorem}{Theorem}[section]

\newtheorem{lemma}[theorem]{Lemma}
\newtheorem{definition}[theorem]{Definition}

\newtheorem*{example*}{Example}

\renewcommand{\paragraph}[1]{\noindent\textbf{#1}}

\newcommand{\prob}{\mathbb{P}}
\renewcommand{\eps}{\varepsilon}

\title{Private Federated Multiclass Post-hoc Calibration}

\author{Samuel Maddock, Graham Cormode, Carsten Maple \\ University of Warwick
}

\iclrfinalcopy %
\begin{document}

\maketitle

\maketitle

\begin{abstract}
    Calibrating machine learning models so that predicted probabilities better reflect the true outcome frequencies is crucial for reliable decision-making across many applications. 
    In Federated Learning (FL), the goal is to train a global model on data which is distributed across multiple clients and cannot be centralized due to privacy concerns. 
    FL is applied in key areas such as healthcare and finance where calibration is strongly required, yet federated private calibration has been largely overlooked. 
    This work introduces the integration of post-hoc model calibration techniques within FL. Specifically, we transfer traditional centralized calibration methods such as histogram binning and temperature scaling into federated environments and define new methods to operate them under strong client heterogeneity.
    We study (1) a federated setting and (2) a user-level Differential Privacy (DP) setting and demonstrate how both federation and DP impacts calibration accuracy.
    We propose strategies to mitigate degradation commonly observed under heterogeneity and our findings highlight that our federated temperature scaling works best for DP-FL whereas our weighted binning approach is best when DP is not required.

\end{abstract}

\section{Introduction}\label{sec:intro}

Federated Learning (FL) is a decentralized approach to machine learning that allows models to be trained across distributed devices without requiring private data to be collected in a single location \citep{mcmahan2016federated, kairouz2019advances}. Rather than sending raw data to a server, clients perform local work and share model updates, such as gradients. %
This paradigm enables organizations to train models without having users send sensitive data and is well-suited to distributed environments where privacy is critical, such as mobile applications \citep{hard2018federated, xu2022training}, healthcare systems \citep{xu2021federated}, and financial services \citep{long2020federated}. However, if applied in isolation, FL is not sufficient to protect the privacy of client data since it is vulnerable to attacks that leak sensitive information \citep{fowl2021robbing}. As such, it is commonly combined with Differential Privacy (DP) which adds carefully calibrated statistical noise into the training process \citep{dwork2006calibrating}. %

Most work in the federated setting has focused on aspects of model training, such as advancing federated optimization algorithms \citep{wang2019matcha, karimireddy2020scaffold, wu2020safa}, or improving the communication efficiency of training \citep{alistarh2017qsgd, chen2021communication}. However, in practical scenarios, model training is a small part of the overall ML pipeline. In many cases, federated deployments require additional solutions for the monitoring and improvement of existing models such as feature selection \citep{banerjee2021fed}, computing federated
evaluation metrics %
\citep{cormode2022federated} and post-processing for fairness \citep{chen2023privacy}.

In this work, we focus on the task of \textit{post-hoc classifier calibration in the federated setting}. Calibration refers to the process of adjusting the output probabilities of a model to better reflect the true likelihood (confidence) of the predicted outcomes. Many applications of FL need to be able to trust model predictions as reliable estimates and thus require calibration: medical diagnosis, risk assessment scenarios (e.g., insurance) and finance (e.g., fraud detection). 
However, it is observed that modern neural networks (NNs) are typically not well-calibrated. 
That is, they are often extremely over or under-confident in their predictions \citep{guo2017calibration}. 

We focus on calibration that is performed post-hoc once the model is trained. 
\red{A post-hoc} calibration model takes the output probabilities (or logits) of a trained %
classifier and learns a map to produce calibrated probabilities. %
One of the simplest approaches to (non-federated) post-hoc calibration is \textit{histogram binning} \citep{zadrozny2002transforming}. %
More sophisticated calibration methods involve scaling the final output logits of the neural network such as \textit{temperature scaling}  \citep{guo2017calibration}. 
Further extensions adapt scaling approaches to be better suited to multi-class classification \citep{kull2019beyond, zhao2021calibrating}.
Some central calibration methods integrate directly into model training~\citep{mukhoti2020calibrating, jung2023scaling, marx2024calibration}.
However, we study post-hoc calibration which is more flexible, and allows calibrating pre-trained models to data in federated settings.

The challenge of federated calibration, with and without Differential Privacy (DP), remains largely unexplored. 
\citet{bu2023convergence} examine model calibration in the central setting and demonstrate that training models with DP exacerbates calibration error, highlighting the importance of calibration for private models.
\citet{peng2024fedcal}~investigate scaling methods for a federated setting under strong data heterogeneity but crucially do not address how to incorporate  DP. 
The most closely related work is due to \citet{cormode2022federated}, who motivate DP-FL calibration. They describe a private histogram binning approach, %
but their method is restricted to binary classifiers and operates under \textit{example-level DP}. In contrast, federated models are typically trained via the stricter notion of \textit{user-level DP} and are often multi-class, which is the more challenging version to tackle. 
Their work does not address the impact of heterogeneity, which we study and design novel techniques for. 

\noindent
\textbf{Our contributions.}
We propose novel methods to perform model calibration in federated heterogeneous environments, 
analyzing both private and non-private variants.
For multiclass tasks, measures such as expected classwise error (ECE) can oversimplify the problem \citep{kull2019beyond}, so we seek to minimize the \textit{classwise} calibration error (cwECE) for multi-class classification, with emphasis on ensuring user-level Differential Privacy (DP) in the private setting. 
Table \ref{tab:related_work} summarizes our work in comparison to related literature. 
\uline{We are the first to handle cwECE in the most challenging non-IID user-level DP-FL settings}, in contrast to prior studies that study IID binary classifiers or omit DP.
Our main findings are as follows:

\textbf{1.} We show naively applying existing methods to non-IID FL settings does not solve the problem, causing model accuracy to decrease after calibration. \uline{We propose two novel frameworks for federated post-hoc calibration} which alleviate these issues, \textbf{FedBinning} (Alg \ref{alg:fedbin}) and \textbf{FedScaling} (Alg \ref{alg:fedscale}). 

\textbf{2.} \red{In the FL setting,} motivated by the failures of naive methods, we modify federated binning via a weighting scheme that is applied by the server with no additional overhead. 
We conclude that \uline{our tailored enhancements outperform scaling approaches under strong heterogeneity}. 
Extending our analysis to the DP-FL setting, we create user-level DP variations and demonstrate that \uline{our proposed scaling methods outperform binning} due to the noise and clipping required for user-level DP binning. 

\textbf{3.} Our findings are validated by  experiments on seven benchmark datasets under two forms of data heterogeneity. We see that \uline{our approaches perform calibration effectively in heterogeneous federated settings}. %

\section{Related Work}\label{sec:related_work}

\begin{table*}[t]
    \caption{Our work compared to related private or federated calibration literature.\label{tab:related_work}}
    \centering
    \small
    \begin{tabular}{llllllllll}
    \toprule
    Name & Federated & Differential Privacy & Multiclass & cwECE & Non-IID \\
    \hline \\ 
    \cite{cormode2022federated} & \checkmark & \checkmark (Example-level) & $\times$ & $\times$ &  $\times$ \\
    \cite{bu2023convergence} & $\times$ & \checkmark (Example-level) & \checkmark & \checkmark & $\times$ \\
    \cite{peng2024fedcal} & \checkmark & $\times$ & \checkmark & $\times$ & \checkmark \\
    Our Work & \checkmark & \checkmark (User-level) & \checkmark & \checkmark & \checkmark \\
    \bottomrule
    \end{tabular}
    \vspace{-18pt}
\end{table*}

\paragraph{Model Calibration.} Classifier calibration is a well-studied area for neural networks. \citet{guo2017calibration} highlight that common model architectures tend to be extremely over or under-confident in their predictions. Many approaches to post-hoc calibration exist, including classical approaches such as Platt scaling \citep{platt1999probabilistic}, isotonic regression \citep{chakravarti1989isotonic} and histogram binning \citep{zadrozny2002transforming} with extensions such as Bayesian averaging via BBQ \citep{naeini2015obtaining}. \citet{guo2017calibration} show that temperature scaling, a simple approach that scales the output logits of the neural network, greatly reduces calibration error against a variety of post-processing calibrators . %
A recent line of work proposes an alternative to post-hoc calibrators via train-time augmentations to the training loss. 
These modifications usually include a regularizer that aims to minimize a proxy for calibration error during model training. 
Examples include \citet{mukhoti2020calibrating} who use the focal loss and \citet{marx2024calibration} who add a regularizer based on distribution matching via Maximum Mean Discrepancy (MMD). %
\red{We consider train-time methods out of scope, as there are no DP-FL approaches and our focus is strictly on how to federate post-hoc calibrators.}

\paragraph{Privacy and Calibration.} DP is known to cause properties of NNs to worsen, such as accuracy~\citep{sander2023tan} and fairness~\citep{cummings2019compatibility, bagdasaryan2019differential}, where the noise addition and clipping from DP-SGD \citep{abadi2016deep} can worsen biases already present in the dataset. 
The only prior work that studies central privacy and its effects on calibration is that of \citet{bu2023convergence}. %
They show models trained via DP-SGD have increased calibration error due to the noise addition and gradient clipping. They propose a %
centrally private temperature scaling algorithm trained via DP-SGD as an extra post-processing step after model training, and demonstrate this effectively lowers the calibration error.

\paragraph{Federated Calibration.} Few works have studied model calibration in FL. \citet{luo2021no} consider the problem of developing prototypes to calibrate local loss functions to avoid heterogeneity and improve model convergence. \cite{lee2024improvinglocaltrainingfederated} apply logit chilling in the federated setting via local temperature scaling as a way to improve convergence speeds. Crucially, both works do not explicitly measure the calibration error of the resulting federated models. 
Recently, \cite{peng2024fedcal}~looks at post-processing calibrators in a heterogeneous federated setting. 
They %
propose FedCal, a scaling approach that trains a local multilayer perceptron (MLP). 
They evaluate the top-label ECE of the final model whereas we argue that classwise-ECE is better suited for federated multi-class problems and study the problem under privacy. 
Last, \citet{chu2024unlocking}
regularize the local loss function for calibration in federated training. 
Their approach relies on estimating the similarity between local and global models. 
In contrast to our work, they do not study post-hoc calibrators or private calibration.

\paragraph{DP-FL Calibration.} Calibration under both DP and FL is mostly unexplored. 
The closest work is by \citet{cormode2022federated} who study model calibration and evaluation metrics in DP-FL. 
However, they restrict to binning methods for binary classifiers, with example-level DP and do not study client heterogeneity (i.e., from label-skew) and client subsampling. 
In contrast, we focus on multi-class classifiers calibrated in a heterogeneous federated setting with user-level DP.

\section{Preliminaries}\label{sec:prelim}

\paragraph{\red{Differential Privacy \& Federated Learning.}}
\label{sec:fl}
We assume a horizontal federated setting where $K$ distributed participants (e.g., mobile devices) each have a local dataset $D_1, \dots, D_K$ with the full dataset denoted $D := \cup_k D_k$. We also assume the local datasets exhibit some form of heterogeneity, i.e., the datasets are not IID. 
We describe how we model this empirically in Section \ref{sec:exp}. We consider a multi-class classification setting where client $k$ has $n_k$ data samples with $d$ features and a target variable $y^k_i \in \mathcal{Y}$ where $|\mathcal{Y}| = c$ is the total number of classes. We are interested in training a global model $\theta(\vx; \vw)$ with learned weights $\vw$ to make predictions of the form $f(\vx ; \vw) = \sigma(\theta(\vx; \vw))$ where $\sigma$ is the softmax function. 
Models are trained in the federated setting using algorithms such as 
FedAvg \citep{mcmahan2017communication}. 
At step $t$ of FedAvg training, participants are subsampled with probability $p$ to train their local model with the current global model weights $\vw^t$ via local SGD for a number of epochs to produce client weights $\vw_k^t$. This is sent back to the server which updates the global model via\footnote{This formula assumes the server learning rate is $1$---see Appendix \ref{appendix:fedavg}.
} $\vw^{t+1} = \frac{1}{p\cdot K}\sum_k \vw_k^{t}$ over all provided client weights at round $t$. This is often wrapped inside a lightweight cryptographic secure-aggregation protocol, SecAgg~\citep{bonawitz2017practical}, which allows the server to aggregate $\vw^{t+1}$ without each client revealing their individual $\vw_k^t$.

\red{SecAgg alone is not enough to guarantee privacy on the output of federated computations \citep{fowl2021robbing}.} Differential privacy \citep{dwork2006calibrating} is a formal definition which guarantees the output of an algorithm does not depend too heavily on any one individual. We aim to guarantee user-level $(\eps,\delta)$-DP, where $\eps$ is named the \emph{privacy budget} and determines an upper bound on the privacy leakage of a differentially private algorithm. The parameter $\delta$ is set very small and corresponds to the probability of failing to meet the DP guarantee.

\begin{definition}[$(\varepsilon, \delta)$-DP] 
A randomized mechanism $\mathcal{M}$ is $(\eps, \delta)$-differentially private if for any neighboring datasets $D, D^\prime$ and any possible subset of outputs $S$ we have
    $\prob(\mathcal{M}(D) \in S) \leq e^\eps \prob(\mathcal{M}(D^\prime) \in S) + \delta$
\end{definition}

We assume \emph{user-level privacy}, which says two datasets $D$, $D^\prime$ are adjacent if $D^\prime$ can be formed as the addition/removal of a single user's data in $D$. In FL, to guarantee user-level DP, clients must clip contributions to ensure bounded sensitivity. For DP-FedAvg \citep{mcmahan2017learning}, the model updates sent from clients are clipped to have norm $C$ and therefore global sensitivity $C$. We show how to extend our methods to user-level DP in Section~\ref{sec:dp} with full details in Appendix \ref{appendix:dp}.

\noindent
\paragraph{Threat Model.} We assume an honest-but-curious model where participants do not trust others with their private data, including their model updates. In addition, we assume there is a central server that follows the algorithm exactly. 
We will use a secure-aggregation protocol \citep{bell2020secure} to allow clients and the server to aggregate local updates. 
When enforcing the use of DP, we assume the central server adds the necessary privacy noise to the aggregated output of secure-aggregation.

\paragraph{Calibration.}
We focus on calibration of neural networks for multi-class classification. The goal is to achieve perfect calibration of such a classifier.

\begin{definition}[Perfect Calibration]
Specifically, for labels $Y \in \mathcal{Y} = \{1, \dots, c\}$ and a predictor $\hat P$ with confidence predictions $\hat p$ we define \textit{perfect calibration} as $\prob(\hat Y = Y | \hat P = \hat p) = \hat p, \forall \hat p \in [0,1]$.
\end{definition}
As this requires knowledge of the underlying joint distribution it is impossible to achieve in practice. 
Instead we rely on empirical measurements of how well-calibrated a classifier is.
For a single target class, 
the most common metric is the Expected Calibration Error (ECE) which partitions predicted confidences $\hat p \in [0,1]$ into bins and averages the difference between accuracy and predicted confidence within each bin.
Several works have studied alternative measures of calibration \citep{gupta2021top, blasiok2023smooth}. 
We apply a stricter measure for the multiclass setting called classwise-ECE (cwECE) \citep{kull2019beyond}.
\begin{definition}[Classwise-ECE]
For a fixed partitioning of $[0,1]$ into $B_1, \cdots, B_M$ bins, the cwECE is defined as $\frac{1}{c}\sum_{j=1}^c \sum_{m=1}^M \frac{|B_{m}|}{N} | \hat p_j(m) - y_j(m)|$ where $N$ is the total number of samples, $\hat p_j(m) := \frac{1}{|B_m|}\sum_{i \in B_m} \hat p_{i,j}$ is the average class $j$ confidence of examples in $B_m$ and $y_j(m)$ is the proportion of class $j$ examples in $B_m$. 
\end{definition}

Our focus is on how to federate post-hoc calibrators. Here, the goal is to learn a map $g : \R \rightarrow [0,1]$ as a form of post-processing over the predictions of a model (i.e., logits or probabilities) which outputs calibrated probabilities. The simplest approach is based on histogram binning for binary classification~\citep{zadrozny2002transforming}. The uncalibrated predicted probabilities $\hat p_i$ are partitioned into $M$ bins 
and each bin is remapped to a new confidence score.
Further extensions use Bayesian Binning into Quantiles (BBQ) which averages multiple binning schemes to produce the final calibrator. See Appendix  \ref{appendix:binning} and \ref{appendix:bbq} for full details of these methods. More recently, \citet{guo2017calibration} proposed a class of calibrators known as scaling methods. These involve learning a linear transformation of the output logits of the neural networks (and not raw probabilities). Scaling approaches take the form
$ g(\vz_i) = \sigma(A \vz_i + \vb) $
where $\vz_i$ is the output logits for example $\vx_i$. 

In our federated setting, we assume that each client partitions their local dataset $D_k$ into three sets; train, calibration and test. 
We measure the cwECE of the global model produced from federated training computed over these federated test sets as a measure of global cwECE. This is in contrast to \citet{peng2024fedcal}. who study a global form of top-label ECE.
In practice, we only need clients to have a train and calibration set, which is easily achieved by splitting clients' local data into two.

\section{Federated Calibration}\label{sec:fedcal}
We focus on two classes of methods: histogram binning and scaling.
Section \ref{sec:bin} highlights the issues with naively extending binning to FL and puts forward our weight aggregation strategy to alleviate this. Section \ref{sec:scale} shows how to federate scaling methods, again overcoming difficulties due to heterogeneity. 
Section \ref{sec:dp} proposes extensions to our framework for user-level DP.

\begin{algorithm}[t]
\caption{Federated Multiclass Binning}\label{alg:fedbin}
\begin{algorithmic}
\small
\STATE {\bfseries Input:} Local datasets $D_1, \dots, D_K$, sampling rate $p$, trained model $\theta$, $M$ bins, $T$ rounds, (optional) $(\eps, \delta)$
\FOR{each round $t=1, \dots T$}
    \STATE Server samples participants with probability $p$ to form the participation set $\mathcal{P}_t$
    \STATE Partition $[0,1]$ into $M$ fixed-width bins, $B_m := [\frac{m-1}{M}, \frac{m}{M}]$ 
    \FOR{each client $k \in \mathcal{P}_t$}
        \FOR{each class $j \in [c]$}
            \STATE{Compute local histograms over positive and negative class $j$ examples with confidences $\{\hat p_{i,j}\}_{i \in D_k}$ $$P_j^k(m) := \sum_{i \in D_k} \mathbf{1}\{y_i = j  \wedge \hat p_{i,j} \in B_m\},\quad N_j^k(m) := \sum_{i \in D_k} \mathbf{1}\{y_i \neq j\ \wedge \hat p_{i,j} \in B_m\} $$}
        \ENDFOR
        \STATE Client $k$ sends each local class histogram $\{P_j^k$, $N_j^k\}$ to the server via SecAgg \citep{bell2020secure}
    \ENDFOR
    \STATE{Server aggregates classwise histograms $P_j = \sum_{k \in \mathcal{P}_t} P_j^k, N_j = \sum_{k \in \mathcal{P}_t} N_j^k$}
    \STATE{\textbf{Optional:} Server clips each $\{P_j, N_j\}$ and adds Gaussian noise to guarantee $(\eps,\delta)$-DP, see Sec \ref{sec:dp}}
    \STATE{For each class $j \in [c]$, the server forms the one-vs-all binning calibrator of the form $$\textstyle g_j(\hat p) := \sum_{m} \1\{\hat p \in B_m\} \frac{P_j(m)}{P_j(m) + N_j(m)}$$}
    \STATE{Return $g(\hat \vp_i) := (g_1(\hat p_{i,1}), \dots, g_c(\hat p_{i,c})) / \sum_{j=1}^c g_j(\hat p_{i,j}) $}
\ENDFOR
\end{algorithmic}
\end{algorithm}
\begin{algorithm}[t]
\caption{Federated Scaling}\label{alg:fedscale}
\begin{algorithmic}
\small
\STATE {\bfseries Input:} $K$ participants with data $D_1, \dots, D_K$, sampling rate $p$, $T$ calibration rounds, (optional) $(\eps, \delta)$
\FOR{each global round $t=1, \cdots, T$}
    \STATE Server samples participants with probability $p$ to form the participation set $\mathcal{P}_t$
    \FOR{each client $k \in \mathcal{P}_t$}
        \STATE Client trains a local scaling calibrator
        $g_k(\vz) := \sigma(A\vz + \vb)$ over their logits $\{\vz_i\}_{i \in D_k}$
        \STATE Client sends final local scaling parameters $A_k, \vb_k$ to the server via SecAgg \citep{bell2020secure}
    \ENDFOR
    \STATE{Server aggregates and averages scaling parameters, $A_t = \frac{1}{|\mathcal{P}_t|}\sum_{k \in \mathcal{P}_t} A_k, \vb_t = \frac{1}{|\mathcal{P}_t|}\sum_{k \in \mathcal{P}_t} \vb_k$}
    \STATE{\textbf{Optional:} Server clips $A_t, \vb_t$ and adds DP noise calibrated to $(\eps,\delta)$-DP over $T$ rounds, see Sec \ref{sec:dp}
}
\ENDFOR
\STATE{Return final calibrator of the form $g(\vz) = \sigma(A_T \cdot \vz + \vb_T)$}
\end{algorithmic}
\end{algorithm}

\subsection{Histogram Binning Approaches}\label{sec:bin}
Histogram binning is usually applied in a binary setting ($c=2$ classes). 
The goal is to map the model's predicted probabilities to more accurate confidence estimates based on the actual outcomes observed. 
This process is carried out by partitioning the interval $[0,1]$ into bins and replacing the scores of examples in a bin by the average accuracy over all examples within that bin.
See Appendix \ref{appendix:binning} for a formal treatment of (centralized) histogram binning.
\citet{cormode2022federated} propose a translation of binning to the federated setting. 
Their idea is to build histograms of positive and negative examples over the bins. 
This yields two histograms built from client responses that are combined 
via secure-aggregation \citep{bell2020secure}. 
Once the server receives these histograms it can build the calibrator by computing the accuracies within each bin.

\paragraph{Multiclass binning technique.}
We turn the problem into $c$ one-vs-all calibrators.
That is, we learn a binning calibrator $g_j$ for each class $j$. 
To do so, each client computes and sends $2c$ histograms---the positive and negative histograms for each class. 
The server receives the sum of all these histograms, which define $c$ binary binning calibrators. 
The final prediction is the normalized probability distribution from each one-vs-all calibrator.
We present the framework for federated binning in Algorithm~\ref{alg:fedbin}.

Directly implementing this approach has some highly undesirable properties in the multi-class federated setting.
We refer to this naive baseline as FedBin. 
To demonstrate the drawbacks, we applied FedBin to calibrate a simple CNN trained over CIFAR10 with significant label-skew (see Section~\ref{sec:exp} and Appendix~\ref{appendix:hetero} for full details) with 100 clients and 10 clients participating per-round.
Figure~\ref{fig1_combined} shows that the baseline binning approach cannot simultaneously achieve low cwECE and good accuracy for high skew (small $\beta$). 
Figure \ref{fig1_c} shows the cause of this failure mode.  
It plots the change in average training accuracy on clients' local datasets after they perform local calibration next to the training accuracy change on the global dataset. 
Letting clients train local calibrators when they have significant skew causes overfitting: a big \textit{increase} in (local) training accuracy but a large \textit{reduction} in (global) test accuracy. 
This overfitting causes significant problems when the federated binner is trained over very few rounds, creating a global calibrator that significantly decreases the overall test accuracy, which is an unacceptable tradeoff.

\paragraph{Addressing Heterogeneity.} 
Our solution to this overfitting problem is for the server to use a weighting scheme when combining the client histograms. %
The high-level idea is to have the server output a calibrator whose prediction is a weighted average between the base model's confidence and that of the federated binning calibrator for each class $j$.
We define the weighted calibrator as 
$\tilde g_j(\hat p) = \alpha_j \cdot g_j(\hat p) + (1-\alpha_j) \cdot \hat p $,
where $\alpha_j = \text{clip}(\frac{\tilde N_j}{N_j}, 1)$, $\tilde N_j$ is the total number of class $j$ examples aggregated so far and $N_j$ is the total number of class $j$ examples in the dataset. This balances reliance on the federated calibrator, ensuring stability under label skew or a small number of training rounds.

\paragraph{Instantiating Algorithm \ref{alg:fedbin}: FedBBQ.}
The last part of our solution is for the server to combine information at multiple bin granularities. 
In the central setting, the best binning-based calibrator, BBQ, fuses multiple different histograms, by applying Bayesian averaging
\citep{naeini2015obtaining}. 
We instantiate this averaging in the federated model by materializing a collection of histograms with different bin counts. 
Specifically, 
we build a histogram with $2^M$ total bins (via Algorithm~\ref{alg:fedbin}) and form $M$ different histograms by merging bins. This forms $M$ calibrators which are averaged using the Bayesian weighting of BBQ to produce a final calibrator.
We apply our weighting scheme for addressing heterogeneity to the final output. 
We refer to the final combined protocol as FedBBQ, and give full details in Appendix \ref{appendix:bbq}.
Subsequently, we only show results for FedBBQ, as it outperformed the %
single unweighted histogram binning method FedBin. We fix $M=7$ (128 total bins) as this has consistent results across all settings, based on our ablation study (summarized in Appendix~\ref{appendix:bins}).

\begin{figure*}[t!]
\centering
  \subfloat[Binning: cwECE \& test acc \label{fig1_combined}]{%
       \includegraphics[width=0.28\linewidth]{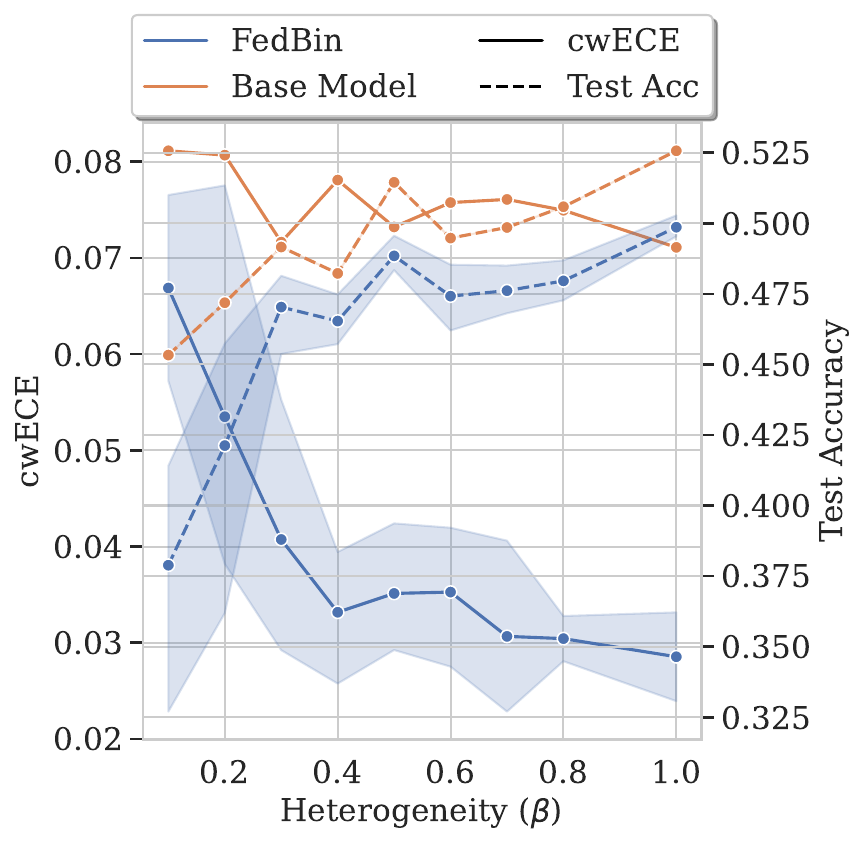}}
  \subfloat[Binning: Local predictions\label{fig1_c}]{%
        \includegraphics[width=0.28\linewidth]{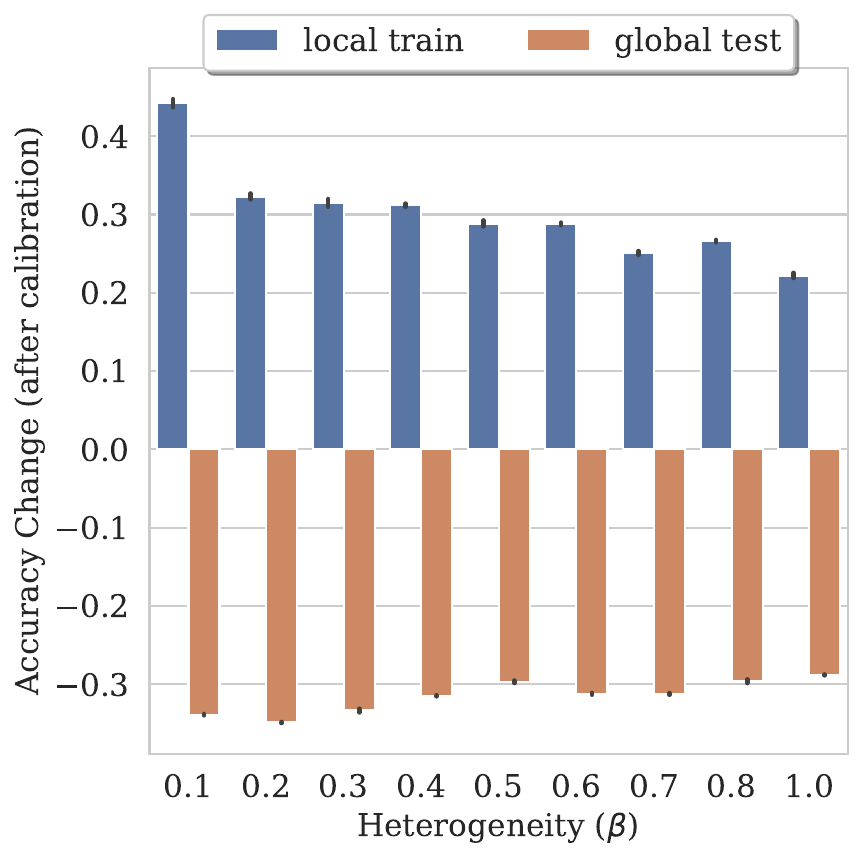}}
  \subfloat[Scaling: Local predictions\label{fig2_c}]{%
       \includegraphics[width=0.28\linewidth]{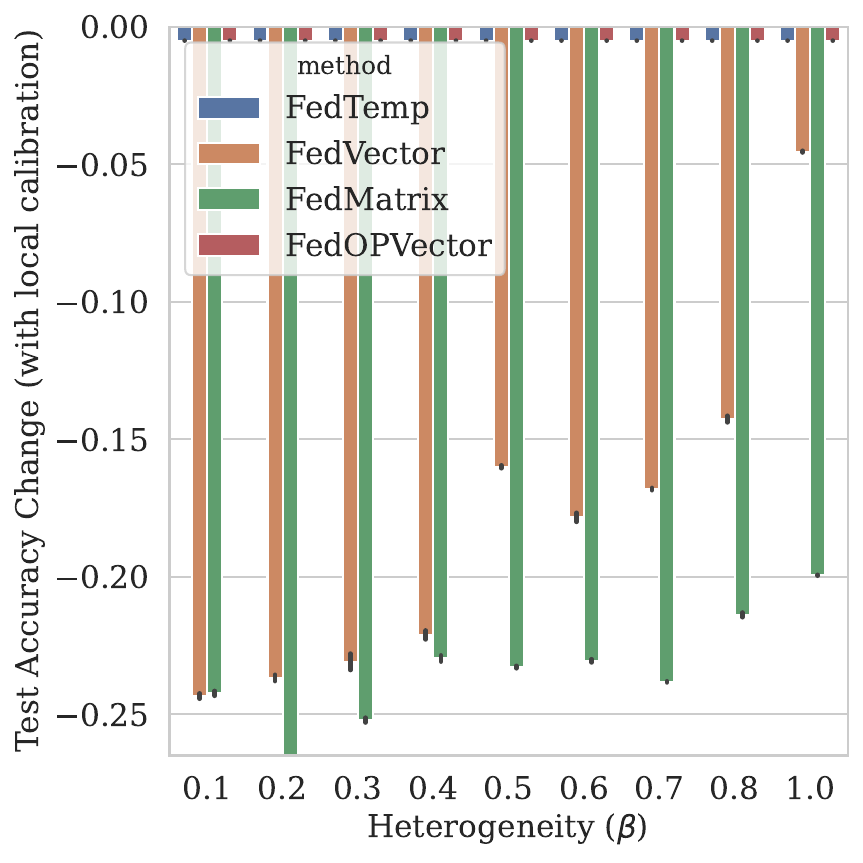}} \\
  \caption{Naive federated calibration on CIFAR10 (Simple CNN), varying heterogeneity $\beta$.\label{fig:1}}
\end{figure*}
\begin{figure*}[t!]
\centering
    \begin{center}
        \includegraphics[width=0.9\linewidth]{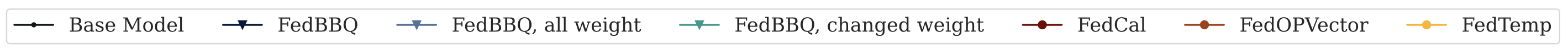} \vspace{-6mm}
    \end{center}
  \subfloat[Varying $T$: cwECE \label{fig3_ece}]{%
    \includegraphics[width=0.28\linewidth]{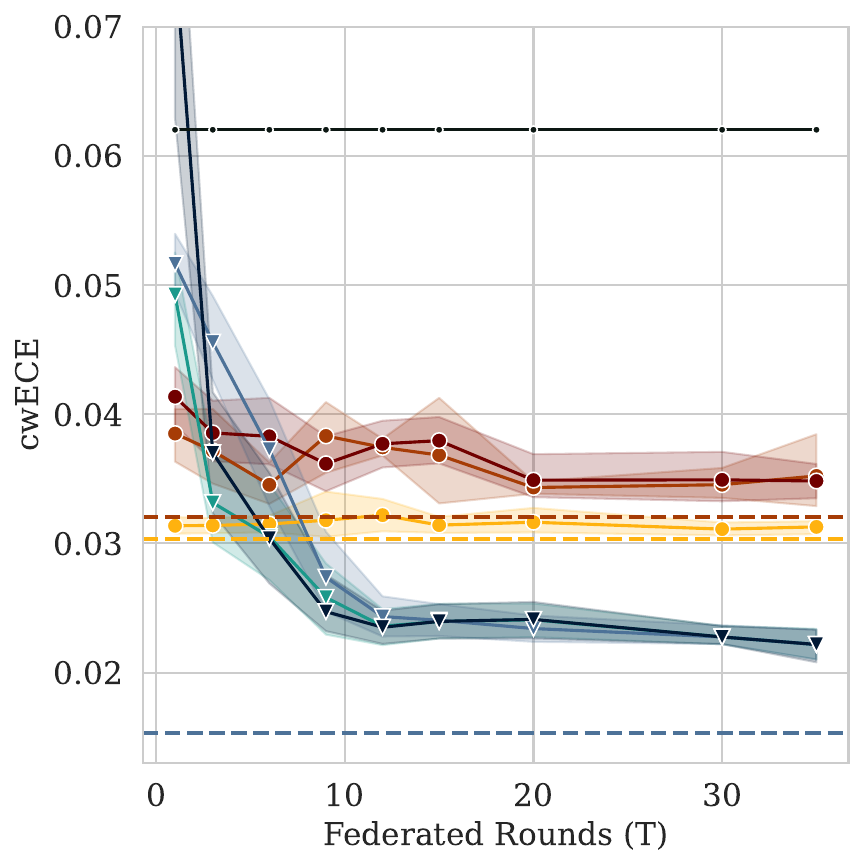}}
  \subfloat[Varying $T$: Test Accuracy \label{fig3_acc}]{%
        \includegraphics[width=0.28\linewidth]{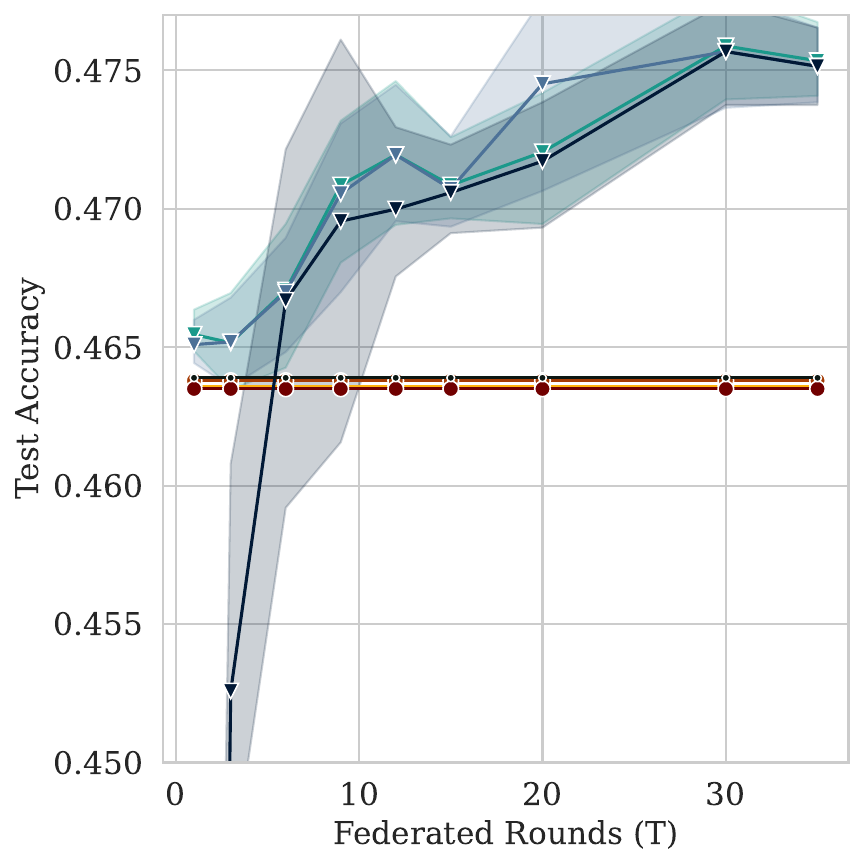}}
  \subfloat[Varying $\beta$: cwECE \label{fig3_beta}]{%
       \includegraphics[width=0.28\linewidth]{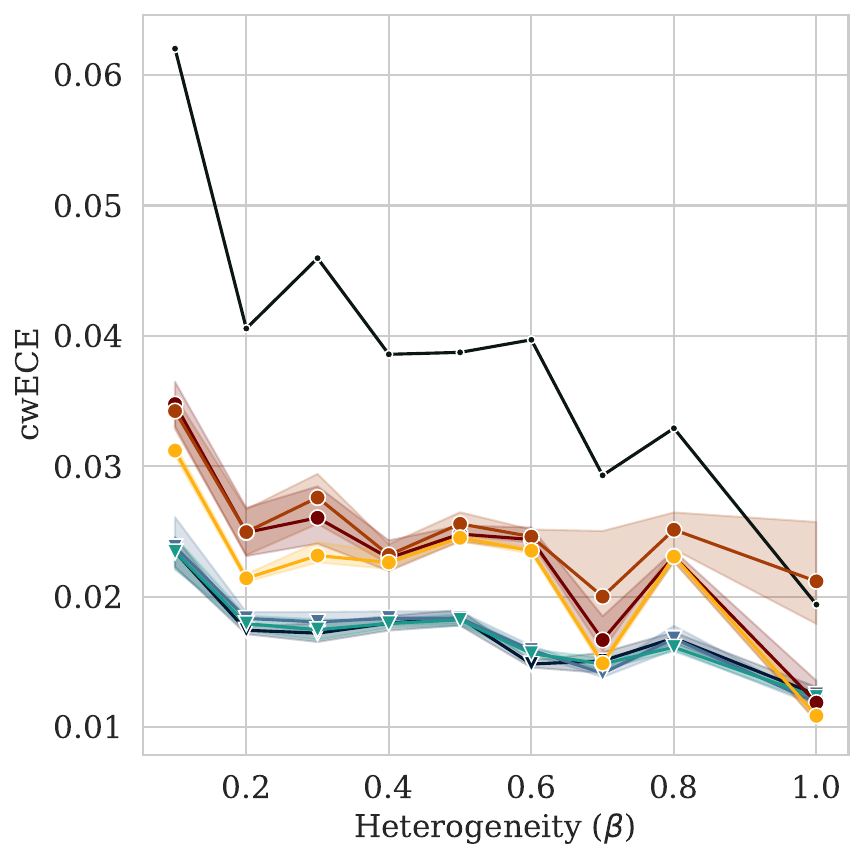}} \\
  \caption{FL Calibration on CIFAR10 (Simple CNN), $\beta=0.1$ unless otherwise stated.}
  \vspace{-4mm}
\end{figure*}
\begin{figure*}[t!]
\centering
    \begin{center}
        \includegraphics[width=0.9\linewidth]{figures/legend.png} \vspace{-6mm}
    \end{center}
  \subfloat[Varying $\eps$: cwECE \label{fig4_ece}]{%
       \includegraphics[width=0.28\linewidth]{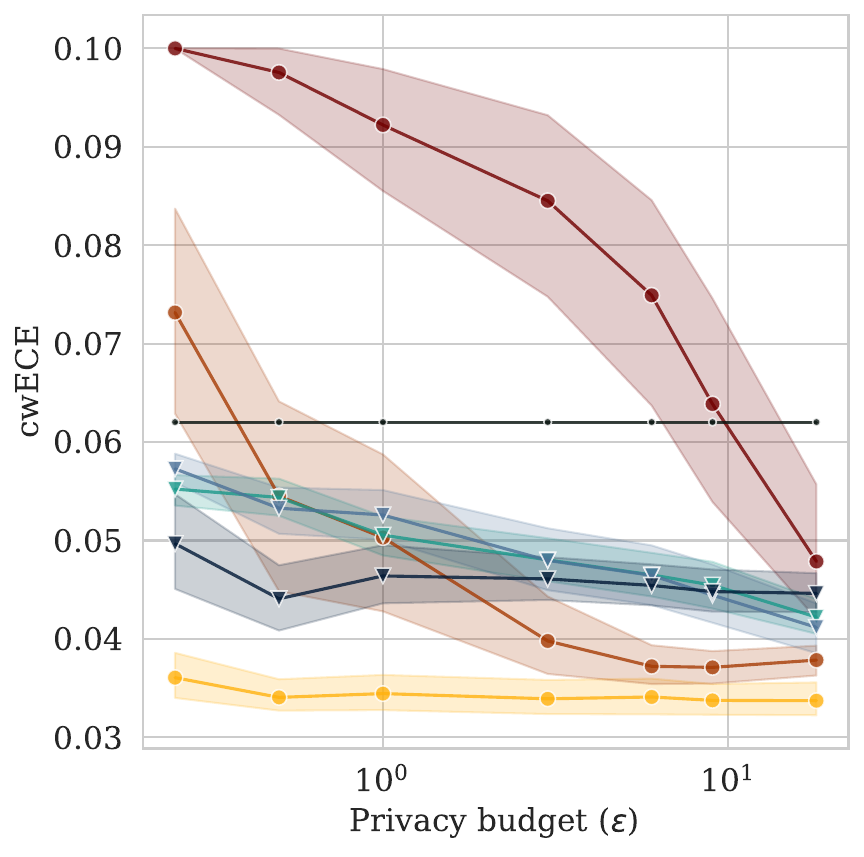}}
  \subfloat[Varying $\eps$: Test Accuracy \label{fig4_acc}]{%
    \includegraphics[width=0.28\linewidth]{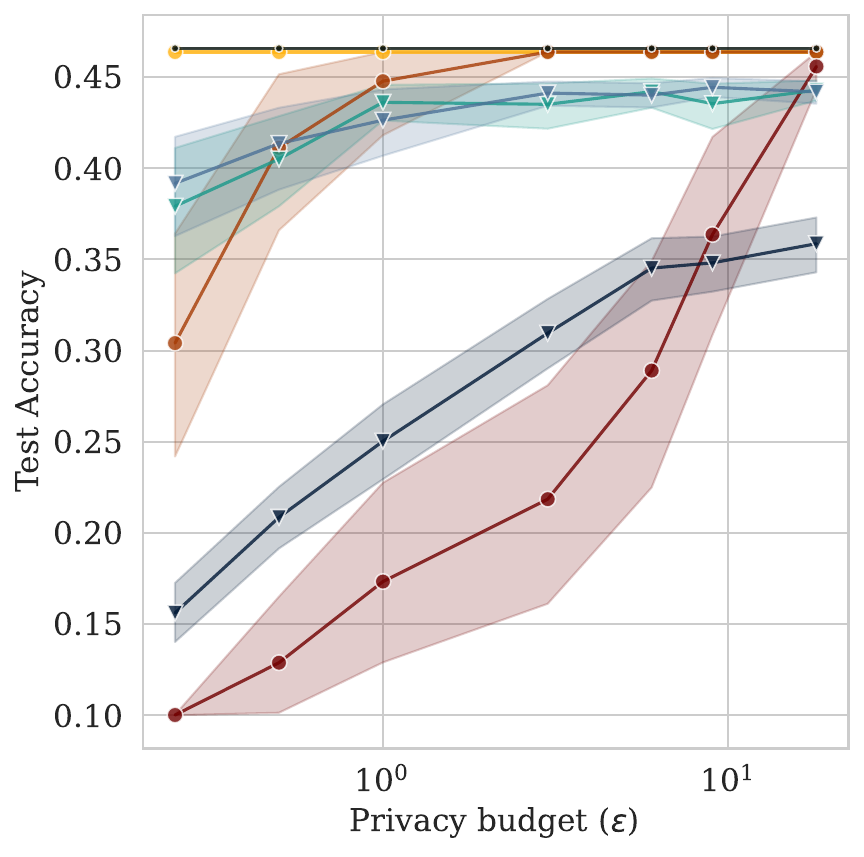}}
  \subfloat[Varying $T$ and $\eps$: cwECE \label{fig4_varyt}]{%
       \includegraphics[width=0.28\linewidth, height=3.8cm]{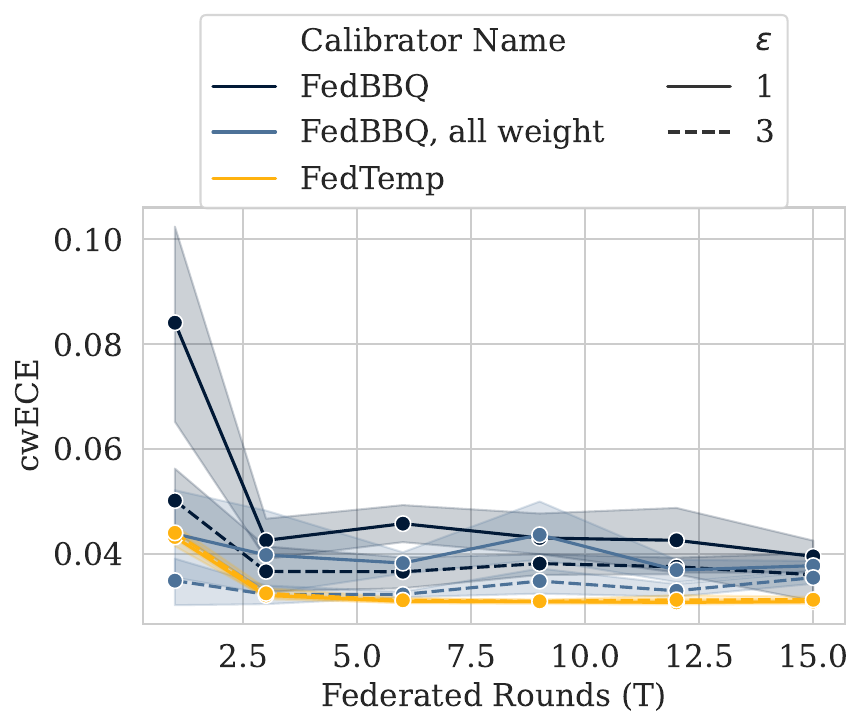}} \\
  \caption{DP-FL Calibration on CIFAR10 (Simple CNN) $\beta=0.1$ varying $\eps$ with $\delta=10^{-5}$\label{fig:4}}
  \vspace{-5mm}
\end{figure*}

\subsection{Scaling Approaches}\label{sec:scale}

We next design federated versions of the scaling approaches introduced by %
\cite{guo2017calibration}. 
This learns a linear transformation of the output logits $g(\vz_i) = \sigma(A\vz_i + \vb)$. Since scaling calibrators are a transformation of logits, they can be viewed as a neural network layer and trained via a FedAvg framework. 
That is, we let clients train their local scaler and have them share the local parameters $A_k, \vb_k$ which the server aggregates over multiple rounds. The global calibrator can be updated using the average of these parameters. This \red{framework} is outlined in Algorithm \ref{alg:fedscale}. 

\paragraph{Instantiating Algorithm \ref{alg:fedscale}.}
We define the following scaling approaches: FedMatrix with $A,\vb$ unconstrained; FedVector with $A$ restricted to diagonal entries; and FedTemp where $A=\textbf{1}/a, b=0$. 

However, similar problems arise when naively applying scaling methods as with naive binning. 
Figure~\ref{fig2_c} shows the change in (global) test accuracy after local scaling to highlight issues on non-IID data. As heterogeneity increases ($\beta$ decreasing), naive vector (FedVector) and matrix (FedMatrix) scaling suffer an increasing penalty to the global accuracy.
This local overfitting results in final calibrated models with worse test accuracy than the base model. 
As vector and matrix scaling learn a bias term $\vb$, this shifts class-predictions and changes accuracy. 
There is no accuracy loss for FedTemp (temperature scaling) as it only rescales logits,  preserving the relative ordering of model predictions.

\paragraph{\red{Addressing Heterogeneity.}} 
To prevent accuracy loss, we can enforce that the scaling for calibration does not change the predictions by adopting \textit{order-preserving training}~\citep{rahimi2020intra}. 
Combined with vector scaling, this ensures that the top-1 accuracy is unchanged and so no accuracy degradation occurs.  
We denote this approach as FedOPVector. 
In Figure \ref{fig2_c}, we see that FedOPVector does not suffer any accuracy loss.
In our later experiments, we compare to the FedCal approach from \cite{peng2024fedcal} which uses an MLP to transform the logits, also in combination with order-preserving training \footnote{They propose an additional scheme in a train-time calibration setting which we omit in our post-hoc setting.}. We present more details on order-preserving training in Appendix~\ref{appendix:op}.

\subsection{Federated calibration with user-level DP}\label{sec:dp}
Our methods provide DP guarantees by appropriate noise addition and clipping during a federated round of calibrator training. 
We guarantee user-level DP by clipping each user's contributions based on a parameter $C$, ensuring bounded sensitivity. 
Appendix \ref{appendix:dp} gives our formal $(\eps,\delta)$-DP guarantees.

\paragraph{Binning under user-level DP.} 
As federated binning approaches require user-level DP histograms, we apply the methods of \citet{liu2023algorithms}.
One issue we observe is that histograms are particularly sensitive to the chosen clipping norm $C$. 
Such sensitivity to $C$ is exacerbated in histogram binning as the positive and negative example histograms are often over different scales, e.g., there are far more negative examples than positive in a one-vs-all setting. We overcome this by using two separate clipping norms $C_+, C_-$ which determine the clipping norm for the individual histograms. 
A different problem occurs when utilizing FedBBQ over multiple binning schemes, as the privacy cost scales with the number of binners. 
Our solution avoids this, as we instantiate a single large histogram of size $2^M$ and form multiple binning schemes by merging bins. 
Under these settings, the privacy cost scales proportional to $2Tc$ where $c$ is the number of classes and $T$ is the total number of federated calibration rounds (see Appendix~\ref{appendix:dp}).

\paragraph{Scaling under user-level DP.} As the scaling methods are trained under FedAvg, we can apply DP-FedAvg with minimal changes \citep{mcmahan2017learning}. 
Here, the local scaling parameters $A_k,\vb_k$, are clipped to norm $C$ by clients and shared with the server under secure-aggregation \citep{bell2020secure}. 
The server can then  average these aggregated parameters and add calibrated Gaussian noise to guarantee user-level DP. Here, the privacy cost scales proportional to $T$, the total number of federated calibration rounds.

\begin{table*}[t]
    \caption{Comparison of cwECE across methods in a non-DP FL setting, $\beta=0.1$ for non-LEAF datasets and $T=12$. Mean cwECE is reported as a \% alongside standard deviation over 5 runs. Results marked $^*$ suffer $> 1$\% drop in test accuracy after calibration.}
    \label{tab:1}
    \centering
    \small
    \setlength{\tabcolsep}{5pt}
\begin{tabular}{llrrrrr}
\toprule
            &  Base & FedBBQ &FedBBQ &FedCal &FedOPVector&FedTemp \\
Dataset &              &                &  All weight                                       &                &                &                \\
\midrule
CIFAR10 (18) &  8.11 &  \textbf{2.375 (0.276)} &       2.499 (0.36) &   4.391 (0.849) &  3.779 (0.741) &  2.428 (0.112) \\
CIFAR100 (18) &  0.94 &  0.317$^*$ (0.017) &  \textbf{0.29 (0.02)} &  0.526 (0.025) &  0.472 (0.018) &   0.341 (0.01) \\
SVHN (18) &  2.05 &\textbf{1.46 (0.145)}&  1.483 (0.247) & 2.468 (0.165) &   2.479 (0.51) &  2.172 (0.182) \\
Tinyimnet (18) &  0.21 &  0.211$^*$ (0.005) &  0.194 (0.002) &        0.187 (0.001) &  \textbf{0.187 (0.0)} &  0.198 (0.001) \\
FEMNIST (2) &  0.18 &  0.161 (0.005) &  \textbf{0.153 (0.004)} &          0.237 (0.048) &  0.205 (0.014) &  0.182 (0.001) \\
MNIST (2) &  0.77 &  \textbf{0.557 (0.037)} &      0.586 (0.046)  &  1.232 (0.238) &  0.865 (0.202) &  0.763 (0.018) \\
Shakespeare (3) &  \textbf{0.14} &  0.212 (0.001) &  0.2 (0.002) &   0.331 (0.06) &    0.38 (0.04) &  0.165 (0.052) \\
\bottomrule
\end{tabular}
\vspace{-4pt}
\end{table*}

\begin{table*}[t]
    \centering
    \caption{DP-FL comparison, $\beta=0.1$ for non-LEAF datasets, $T$=12 and ($\eps, \delta$) = $(1, 10^{-5})$. %
    Methods in bold achieve lowest cwECE for a particular dataset, ignoring methods that suffer accuracy drops.}
    \label{tab:2}
    \small
    \setlength{\tabcolsep}{5pt}
\begin{tabular}{lllllll}
\toprule
            &  Base &         FedBBQ & FedBBQ  &         FedCal &    FedOPVector & FedTemp \\
Dataset &               &                &  All weight                                         &                &                &                \\
\midrule
CIFAR10 (18) &  8.11 &   4.582$^*$ (1.03) &  \textbf{3.086 (0.15)} &  9.999$^*$ (0.002) &  4.118 (0.812) &  4.423 (0.122) \\
CIFAR100 (18) & 0.94 &  0.129$^*$ (0.005) &  0.132$^*$ (0.008) &      1.0$^*$ (0.0) &  0.592 (0.033) &   \textbf{0.346 (0.006} \\
SVHN (18) &  2.05 &    6.89$^*$ (0.64) &   2.047 (0.002) &          9.955$^*$ (0.034) &  2.396 (0.528) &  \textbf{2.007 (0.025)} \\
Tinyimnet (18) &  0.21 &  0.232$^*$ (0.017) &        0.212 (0.0) &   0.5$^*$ (0.0) &  \textbf{0.192 (0.009)} &  0.198 (0.003) \\
FEMNIST (2) &  0.18 &  1.888$^*$ (0.029) &      0.184 (0.002) &          0.232 (0.025) &  0.201 (0.009) &   \textbf{0.18 (0.001)} \\
MNIST (2) &  0.77 &  9.381$^*$ (0.86) &        0.769 (0.0) &  9.998$^*$ (0.003) &  2.607 (2.153) &  \textbf{0.746 (0.008)} \\
Shakespeare (3) &  0.14 &  1.281$^*$ (0.051) &       0.15 (0.003) &  0.387 (0.121) &  0.339 (0.057) &    \textbf{0.138 (0.0)} \\
\bottomrule
\end{tabular}
\vspace{-4pt}
\end{table*}

\section{Experiments} \label{sec:exp}
We perform federated calibration on a base model trained via FedAvg without DP on 7 datasets: MNIST %
with a 2-layer CNN and ResNet18; CIFAR10/100 %
with a 2-layer CNN and ResNet18; SVHN, 
 Tinyimagenet 
and FEMNIST %
with ResNet18; and Shakespeare 
with an LSTM for next-character prediction. 
See Appendix \ref{appendix:datasets} and \ref{appendix:hp} for dataset and model training details. For FEMNIST and Shakespeare we utilize the LEAF benchmark \citep{caldas2018leaf} which federates in a natural way to induce heterogeneity. For other datasets we assume $K=100$ clients and simulate heterogeneity via the label-skew approach of \citet{yurochkin2019bayesian}. This involves sampling client class distributions from a Dirichlet($\mathbf{\beta}$) where smaller $\beta$ creates more skew (see Appendix \ref{appendix:hetero}). %
Figures on other datasets can be found in Appendix \ref{appendix:experiments}. Our FL models achieve test accuracy that matches or improves over those used in FedCal \citep{peng2024fedcal}, see Table \ref{tab:models}. %

For instantiating our binning framework (Alg \ref{alg:fedbin}) we consider FedBBQ and two variations with our weighting-scheme: \textit{all weight}, applies the weighting scheme to all predictions and \textit{changed weight} which only applies weights to examples that have their class prediction changed by the calibrator. For scaling (Alg \ref{alg:fedscale}), we consider FedTemp and FedOPVector. %
We also compare to a post-hoc variant of FedCal \citep{peng2024fedcal}. %
We focus on cwECE in our experiments, see Appendix \ref{appendix:ece} for a detailed discussion on why this is preferable to ECE.

We are interested in answering the following research questions:
\begin{itemize}
    \item \textbf{RQ1 (Non-IID FL)}: Which method is most effective in a non-IID FL setting?
    \item \textbf{RQ2 (Mitigating Heterogeneity)}: Do our proposed mitigations for heterogeneity prevent accuracy loss in both FL and DP-FL settings?
    \item \textbf{RQ3 (DP-FL)}: Which method achieves the best balance of accuracy and calibration error in a non-IID user-level DP-FL setting?
    \item \textbf{RQ4 (Robustness):} How robust are these findings across different datasets?
\end{itemize}

\paragraph{RQ1 (Non-IID FL):}
In Figure \ref{fig3_ece}, we vary the number of federated rounds $(T)$ used to train the calibrator on CIFAR10 with no DP. We observe FedTemp has consistently good cwECE but that as $T$ increases the FedBBQ approach outperforms it.
This is because binning methods aggregate histograms, and, if given enough rounds, can closely match the performance of a central calibrator (displayed in dashed lines). 
Note that while FedBBQ only outperforms scaling when $T$ is around $30$, this is far smaller than the total number of federated rounds required to train the base FL model itself. Hence, these calibration approaches are lightweight and can be performed as an extra federated epoch at the end of training with modest additional overhead to the overall federated training pipeline. In Appendix \ref{appendix:overhead}, we benchmark client computation and communication overheads and find all methods are lightweight but that BBQ achieves the best balance. We find scaling methods with a larger number of parameters (i.e., FedCal) often perform the worst. \\\red{\textbf{Summary:} For the FL setting, FedTemp is the most competitive scaling method when $T$ is small but is outperformed by FedBBQ when trained over multiple rounds ($T>30$). We recommend using FedBBQ with a moderate number of rounds $(T=30)$ in the non-DP FL setting.}

\paragraph{RQ2 (Mitigating Heterogeneity):} 
In Figure \ref{fig3_beta}, we vary %
$\beta$ on CIFAR10. We fix the number of FL rounds to $T=12$. We observe binning approaches work well under all ranges of heterogeneity as the algorithm aggregates histograms which mostly mitigates skew. We observe FedTemp achieves best cwECE out of the scaling approaches but is outperformed by binning methods. In Figure \ref{fig3_acc}, we fix $\beta =0.1$ and vary $T$. We observe all scaling methods obtain the same accuracy as the base model (all scaling methods in the figure are overlaid). We clearly find our weighting scheme for FedBBQ and the order-preserving scaling methods prevent any loss in test accuracy. In contrast, if the number of rounds is too small, unweighted FedBBQ suffers a large drop in accuracy as discussed in Section \ref{sec:bin}. \\\red{\textbf{Summary:} Our weighting schemes prevent accuracy degradation for FedBBQ and the order-preserving training of FedOPVector prevents the accuracy degradation observed in FedVector in Section \ref{sec:fedcal}. FedBBQ with weighting has better ECE than FedOPVector and so this is further evidence to recommend the use of weighted FedBBQ for non-DP FL settings.}

 \paragraph{RQ3 (DP-FL):} In Figure \ref{fig:4} we study the DP-FL setting, varying the privacy budget ($\eps$) and compare binning to scaling. We set $C=0.5$ for scaling methods and $C=[10,50]$ for binning approaches. 
 Observe FedCal, which utilizes an MLP, is the most sensitive to noise since it has the most model parameters. 
 For any reasonable privacy ($\eps < 10$) it achieves larger cwECE than the base model. We find FedTemp is the most resilient to high noise since it shares only a single parameter under DP-FedAvg. Binning methods struggle under DP-noise due to the user-level clipping on histograms. We explore this further in Appendix \ref{appendix:dp_fl}. In Figure \ref{fig4_varyt}, we vary both the number of rounds ($T$) and the privacy budget ($\eps$) for CIFAR100 (CNN) and plot the cwECE. We study $\eps=1,3$ for FedTemp, FedBBQ and our weighted FedBBQ. We find FedTemp is insensitive to DP noise, with a decrease in cwECE that plateaus quickly as $T$ increases. In contrast, the binning approaches often struggle in a one-shot ($T=1$) setting but closely match FedTemp as $T$ increases, even under DP noise.
\\\red{\textbf{Summary:} In the DP-FL setting, we recommend using FedTemp as it achieves the best cwECE with no accuracy degradation unlike FedBBQ methods which may degrade accuracy under high privacy.}

\paragraph{RQ4 (Robustness):} In Tables \ref{tab:1} and \ref{tab:2} we explore the cwECE across each benchmark dataset in FL and DP-FL settings. %
For non-DP (Table~\ref{tab:1}), we observe binning approaches consistently outperform scaling methods, achieving lowest cwECE overall. On two datasets, FedBBQ degrades model accuracy, something our weighting scheme helps prevent.
For DP-FL (Table \ref{tab:2}), we find FedTemp consistently achieves best cwECE with smaller variance. It is clear that under DP noise methods with fewer parameters perform best (i.e., FedTemp) compared to FedCal or FedBBQ which have a larger number of parameters. We note FedTemp with DP sometimes achieves better cwECE than without DP. We find the clipping involved in DP actually helps mitigate local skew and show this empirically in Appendix \ref{appendix:temp_clip}.
We also flag results that suffer a drop of $> 1\%$ test accuracy. This reveals the only method to suffer a severe drop is the naive unweighted FedBBQ approach (with and without DP), highlighting the necessity of our weighting schemes which prevent accuracy degradation.
\\\red{\textbf{Summary:} Extending our results across multiple datasets, we find FedBBQ remains best in a FL setting under strong heterogeneity. Conversely, for DP-FL we find FedTemp achieves the best balance of accuracy and cwECE. We find large-parameter scalers like FedCal unsuitable in all settings.}

\paragraph{Conclusion.} 
Our results show we can calibrate effectively in the federated model, even in the face of high heterogeneity, and skew, 
 without affecting model accuracy.
Overall, for FL, we find binning with weight adjustments achieves best cwECE if trained over multiple rounds. However, in DP-FL, temperature scaling is preferred as it is resilient to both non-IID skew and DP noise, something binning and higher-order scalers struggle with. 
Thus, we can make clear recommendations for how best to calibrate models in each setting.

\clearpage
{\small
\bibliography{paper}
\bibliographystyle{iclr2026_conference}
}

\appendix
\section{Experiment Details}
\label{appendix:replication}
All experiments (federated model training and federated calibration) were performed on a machine with a Dual Intel Xeon E5-2660 v3 @ 2.6 GHz and 64GB RAM. No GPUs were used in the training or calibration of the federated models. All experiments simulate client training sequentially (i.e., there is no paralellism) and individual federated calibration runs finish within 1 hour with full experiments (with $5$ repeats) taking at most 24 hours.
\subsection{Datasets and Models}\label{appendix:datasets}

\begin{table*}[t]
    \caption{Datasets used in our experiments.\label{tab:datasets}}
    \centering
    \small
    \begin{tabular}{lllll}
    \toprule
    Dataset & Total Samples ($|D|$) & Number of clients ($K$) & Heterogeneity Type & Model Arch \\
    \hline \\ 
    MNIST & 70,000 & 100 & Label-skew & 2-layer CNN\\
    CIFAR10 & 70,000 & 100 & Label-skew & ResNet18  \\
    CIFAR100 & 70,000 & 100 & Label-skew & ResNet18 \\
    SVHN & 70,000 & 100 & Label-skew & ResNet18 \\
    Tinyimagenet & 100,000 & 100 & Label-skew & ResNet18 \\
    FEMNIST & 805,263 & 1778 & LEAF & ResNet18 \\
    Shakespeare & 4,226,15 & 607 & LEAF & LSTM \\
    \bottomrule
    \end{tabular}
    \vspace{-4pt}
\end{table*}

We use a variety of datasets in our experiments, as described in Table \ref{tab:datasets}. In all experiments we merge the train and test sets to create a single global dataset. This dataset is then federated to clients via the label-skew approach described in Appendix \ref{appendix:hetero}. Each client’s local dataset is split into a train, test and calibration set. We take 80\% of the dataset as training and 10\% each for testing and calibration. 
All base models are also trained in a federated setting. 
In more detail:
\begin{itemize}
    \item \textbf{MNIST} \citep{deng2012mnist} is an image classification dataset that classifies handwritten numerical digits between 0-9. It has 10 classes. We train both a simple CNN and a ResNet18 model in our experiments. We partition the full dataset of 70,000 samples (train + test) across 100 clients. In the IID setting this results in each client having 560 train samples, 70 test and 70 calibration. 
    \item \textbf{CIFAR10/CIFAR100} \citep{krizhevsky2009learning} is an image classification dataset with 10/100 classes. We train both a simple 2-layer CNN and a ResNet18 model in our experiments.  We partition the full dataset of 70,000 samples (train + test) across 100 clients. In the IID setting this results in each client having 560 train samples, 70 test and 70 calibration. 
    \item \textbf{SVHN} \citep{svhn} is an image classification dataset with $10$ classes formed from house number digits. We partition the full dataset of $70,000$ samples across $100$ clients using the label-skew approach. In the IID setting, each client has 560 train samples, 70 test and 70 calibration. 
    \item \textbf{Tinyimagenet (a.k.a. Tinyimnet)} \citep{le2015tiny} is a subset of imagenet with 200 classes, 500 of each.  We partition the full dataset using the label-skew approach but restrict $\beta \ge 0.3$ as setting $\beta$ very small can cause sampling issues in the label-skew partitioning procedure due to the large number of classes.
    \item \textbf{FEMNIST} \citep{caldas2018leaf} is an image classification dataset with 62 classes. We use LEAF which provides a partition of the global dataset to clients where each client’s local dataset has characters with similar handwriting to induce local heterogeneity. We use the “sample” partitioning which assigns user’s local samples into train and test groups. We then take 50\% of the local test set to form a client’s calibration set. We filter away any users with fewer than 100 samples, resulting in 1,778 clients with an average of 182 samples for training, 23 for test and 23 for calibration.
    \item \textbf{Shakespeare} \citep{caldas2018leaf} is a next-character prediction task trained on the works of Shakespeare. We use LEAF which provides a partition of the global dataset to induce heterogeneity. In this case, each speaking role, in each  of Shakespeare's works, is a client. We use the “sample” partitioning which assigns user’s local samples into train and test groups. We then take 50\% of the local test set to form a client’s calibration set. We filter away any users with fewer than 100 samples (i.e., roles with $< 100$ characters) resulting in 607 clients with an average of 5,590 samples for training, 293 for test and 293 for calibration.
\end{itemize}

We use the following model architectures:
\begin{itemize}
    \item \textbf{Simple CNN:} We train a small convolutional neural network with 2 convolutional layers and 2 fully-connected layers. Our architecture follows that of the baseline used by \citet{caldas2018leaf}\footnote{\url{https://github.com/TalwalkarLab/leaf/blob/master/models/femnist/cnn.py}}.
    \item \textbf{ResNet18:} We use the ResNet18 model on CIFAR100, SVHN, FEMNIST and Tinyimagenet in our experiments. This is consistent with its use by \citet{peng2024fedcal} for federated calibration experiments.
    \item \textbf{Stacked LSTM}: For Shakespeare, we follow the model architecture used by \citet{caldas2018leaf} which is a stacked 3-layer LSTM. See the open-source implementation of LEAF for more details\footnote{\url{https://github.com/TalwalkarLab/leaf/blob/master/models/shakespeare/stacked_lstm.py}}.
\end{itemize}
\subsection{Modeling Heterogeneity}\label{appendix:hetero}

In our experiments, we use two methods to partition our benchmark datasets into federated splits that exhibit  heterogeneity.

\textbf{Label-skew}: In experiments where we vary heterogeneity we follow the approach outlined by \citet{yurochkin2019bayesian}. The label skew is determined by  $\pmb{\beta} = \beta \cdot \mathbf{1}^K$, where $K$ is the total number of clients. For each class value $j \in [c]$, we sample the client distribution via $\pmb{p}_j \sim \text{Dirichlet}(\pmb{\beta})$ where a smaller $\beta$ value creates more skew. For a particular client $k \in [K]$ we assign rows with class value $j$ proportional to $p_{j,k}$. This produces a partitioning of the dataset into clients that are skewed via the class values of the dataset. The parameter $\beta$ controls the strength of label-skew, where a larger $\beta$ decreases the skew and reduces heterogeneity. 
The setting of $\beta = 1$ corresponds to performing IID sampling for the client partitioning.

\textbf{LEAF:} For two of our datasets, FEMNIST and Shakespeare, we use a pre-partitioned split from the LEAF benchmark \citep{caldas2018leaf}. 
This federates data in a natural way to create heterogeneity in local datasets. As an example, FEMNIST is partitioned so that each user has digits written by the same hand.

\begin{table*}[t]
    \caption{Base model test accuracies for Simple CNN trained under FedAvg with $\beta \in [0.1, \dots, 0.8]$ \label{tab:models}}
    \centering
    \small
    \begin{tabular}{llllllllll}
    \toprule
    Dataset / $\beta =$ & 0.1 & 0.2 & 0.3 & 0.4 & 0.5 & 0.6 & 0.7 & 0.8 & 1 (IID) \\
    \hline \\ 
    MNIST & 93.6\%	& 93.7\%	& 94.6\% & 95.3\% & 95\% &95.4\%	& 96.0\% &	96.5\% & 96.9\% \\
    CIFAR10 & 46.3\%	& 53.6\% & 51.8\% & 53.8\% & 54.2\% & 55.0\% &	56.3\% & 55.1\% & 63.0\% \\
    CIFAR100 & 23.1\% & 24.1\% & 24.4\% & 24\% & 24.5\%	& 24.9\% & 26.8\% & 24.7\% & 25.5\% \\
    \bottomrule
    \end{tabular}
    \vspace{-4pt}
\end{table*}

\subsection{Hyperparameters}\label{appendix:hp}
\subsubsection{Model Training}

For all of our calibration experiments we calibrate a base model that has been trained on clients' local training datasets via FedAvg \citep{mcmahan2016federated}. In datasets with label-skew, we train a model for each value of $\beta$ as this corresponds to a new partitioning of data to clients. In Table \ref{tab:models}, we report the test accuracies of a federated simple CNN model on MNIST, CIFAR10 and CIFAR100 whilst varying heterogeneity ($\beta$). In general, models trained on federated splits with lower values of $\beta$ (corresponding to more local skew) have worse accuracy. We have comparable test accuracy to the models used by ~\citet{peng2024fedcal} and so are suitable for federated calibration. For example, at $\beta=0.1$ Peng et al. achieve 81\% on MNIST (vs. ours at 94\%), 48\% on CIFAR10 (vs. ours at 46\%) and 21\% on CIFAR100 (vs. ours at 23\%). We list the specific training hyperparameters used for each dataset and model architecture. Specifically, these are the local client learning rate $\eta_C$, the server learning rate $\eta_S$, the local batch size $B$ and number of global epochs $E$. In more detail:

\textbf{CIFAR10/100}: We use 100 clients and sample $10\%$ per-round. We use a local client learning rate of $\eta_C = 0.01$ and a server rate of $\eta_S = 1$, and local batch size $B=64$. For the simple CNN we train for $E=300$ global epochs. For ResNet18 we train with $E=100$ keeping the parameters the same as the above. We achieve an overall test accuracy of $63\%$ in the IID case. For CIFAR100, the parameters are the same as CIFAR10 except the number of global epochs. For the simple CNN we set $E=500$ and for ResNet18, $E=100$. We achieve an overall test accuracy of $25\%$ in the IID case.

\textbf{MNIST}: We train for $E=10$ epochs with a local learning rate of $\eta_C = 0.001$ and server rate $\eta_S = 1$ and $B=64$ for the simple CNN model. We achieve an overall test accuracy of $97\%$ in the IID case.

\textbf{SVHN}: We use the same parameters as MNIST and train for $E=10$ epochs with a ResNet18 model

\textbf{Tinyimagenet:} We use the same parameters as MNIST except train for $E=30$ epochs.

\textbf{Shakespeare}: We take 60 clients per round. We train with $\eta_C = 0.01, \eta_S=0.524$ and $E=10$. We achieve an overall test accuracy of $36\%$ on our LEAF partition.

\textbf{FEMNIST}: We take 60 clients per round. We take $E=10, \eta_C=0.01, \eta_S=1$ and $E=30$. We achieve an overall test accuracy of $83.8\%$ on our LEAF partition.
 
\subsubsection{Federated Calibration}

\textbf{Global rounds ($T$)}: This determines the number of federated rounds that are used to train the calibrator. Note that this is not equivalent to a global epoch during model training, instead $T=1$ is a single step i.e., a single federated round of client participation. In the main paper we explore how $T$ effects calibration performance. We find that most methods perform best when the calibrators have been trained for a few rounds (i.e., $15 < T < 30$) but that some methods like FedTemp also perform well in one-shot settings $(T=1)$.

\textbf{Clipping norm ($C$)}: When using DP we apply a clipping norm in order to guarantee bounded sensitivity. For scaling methods, this involves clipping the model update that is sent to the server for each scaling parameter. We clip the scaling parameters to have norm $C$. We found in our experiment that $C < 1$ gave best results and fix $C=0.5$ in our DP-FL experiments. We explore the use of temperature clipping in Appendix \ref{appendix:temp_clip} in a non-DP setting and find it can have positive improvements in mitigating label-skew. For binning methods we have two clipping norms, $[C_+, C_-]$, one for the positive histogram and one for the negative. We explore the choice of clipping norm for FedBBQ in Appendix \ref{appendix:bin_clip} and find that class-positive clipping norm $C_+$ should be chosen to be relatively small for best results.

\textbf{FedBBQ Binning Parameter ($M$)}: This controls the size of the histogram created in FedBBQ resulting in total size $2^M$. In our experiments we fix this to $M=7$. In Appendix \ref{appendix:bins} we explore values of $M$ and find, although not optimal, it gives consistently good results across all settings.

\textbf{Privacy budget $(\epsilon)$}: This is the differential privacy budget. We calibrate the noise needed via zCDP accounting \citep{bun2016concentrated} see Appendix \ref{appendix:dp} for more details.

\begin{figure*}[t!]
\centering
  \subfloat[Varying $M$ ($\eps=\infty$) \label{appendix:m_m}]{%
    \includegraphics[width=0.32\linewidth]{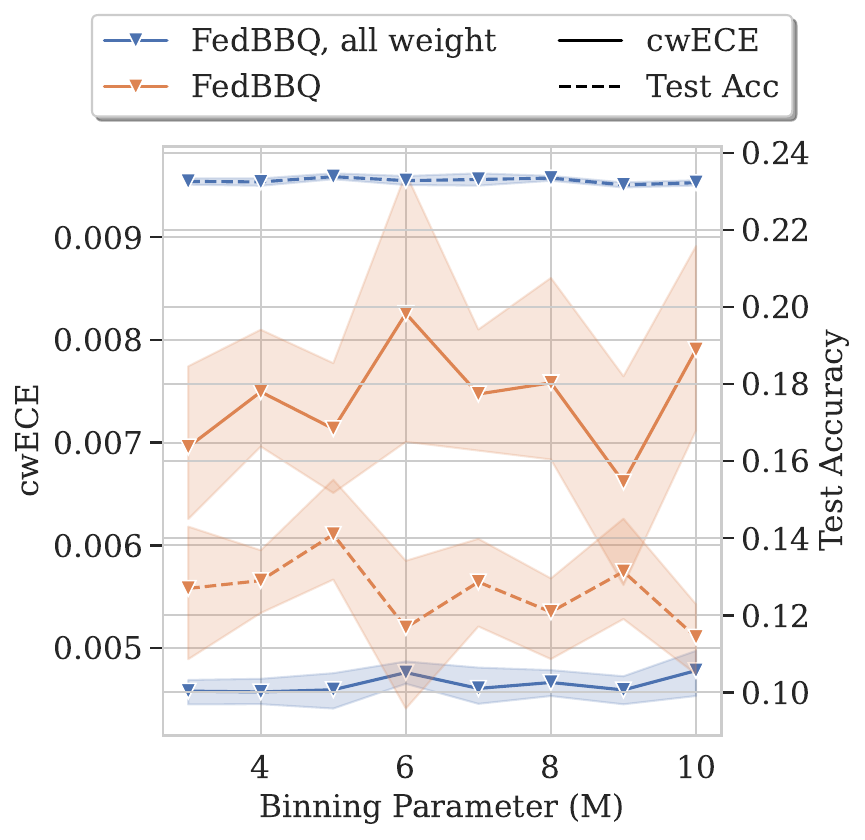}}
  \subfloat[Varying $T$ and and $M$ ($\eps=\infty$) \label{appendix:m_T}]{%
        \includegraphics[width=0.33\linewidth]{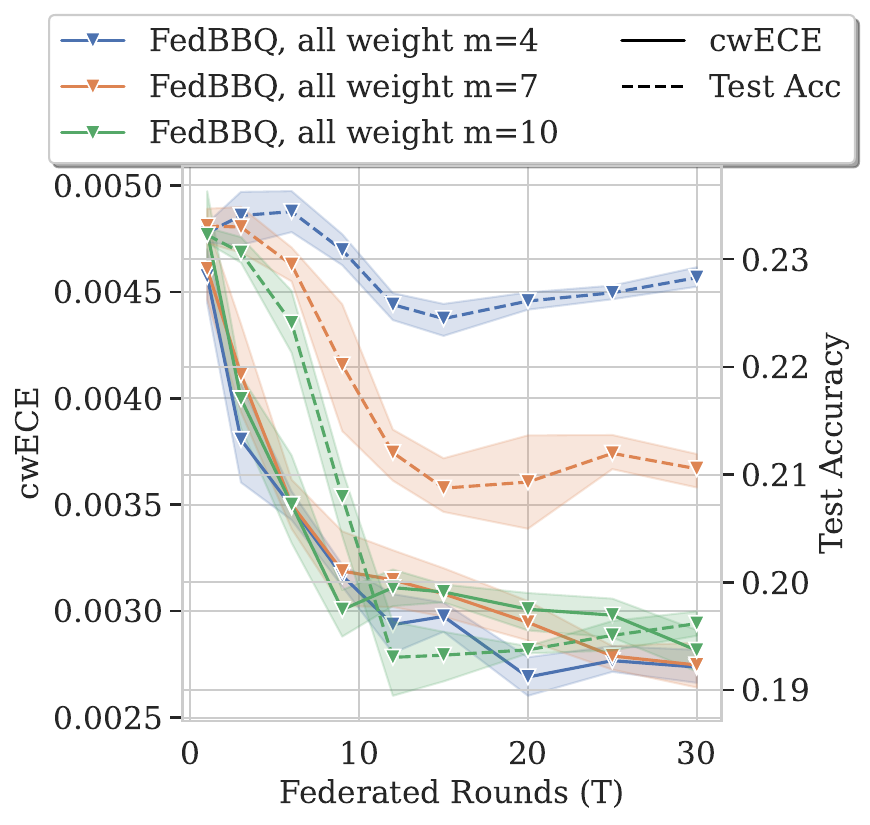}}
  \subfloat[Varying $T$ and $M$ ($\eps=3$) \label{appendix:m_dp}]{%
        \includegraphics[width=0.32\linewidth]{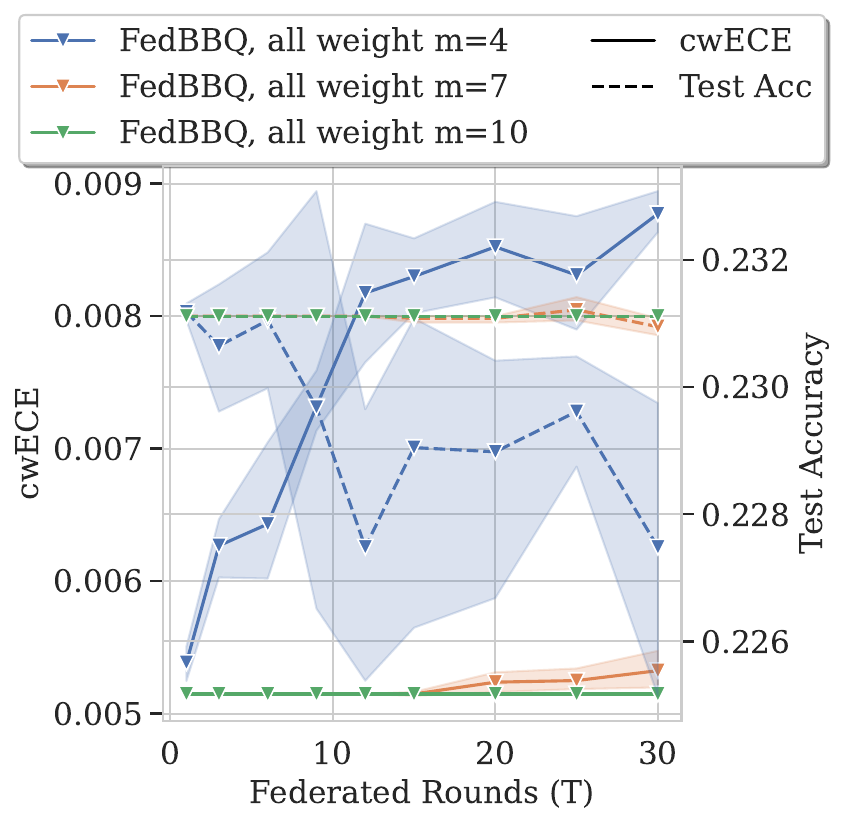}}
  \caption{Federated Calibration via FedBBQ on CIFAR100 whilst varying the bin parameter $M$. This controls the total bins for the BBQ histogram as $B = 2^M$ total bins. \label{appendix:m}}
\end{figure*}

\subsubsection{Varying FedBBQ Binning Parameter ($M$)}
\label{appendix:bins}
In Section \ref{sec:bin} we instantiate BBQ via Algorithm \ref{alg:fedbin}. We instantiate the total number of bins in BBQ via a parameter $M$ which, for each class, creates two histogram over class positive and negative samples with a total number of bins equal to $2^M$. Further histograms are created by merging neighboring bins to create multiple binning calibrators with bin sizes in the range of $\{2, 4, \cdots , 2^{M-1}, 2^M\}$ which are then averaged via BBQ. 
In our main experiments we use a fixed choice of  $M=7$. 
In Figure \ref{appendix:m}, we study the effect of $M$ on both the cwECE and test accuracy after calibration with FedBBQ, and demonstrate that it does not have a significant impact on the results. 

In Figure \ref{appendix:m_m} we vary the binning parameter $M$ with $T=1$ rounds of BBQ in a non-DP setting with heterogeneity parameter $\beta = 0.1$ on CIFAR 100. We plot both the cwECE and the test accuracy and compare our FedBBQ all weight method against naive FedBBQ. We observe that our weighting approach has much less variance across all $M$ values with consistent test accuracy and cwECE which outperforms naive BBQ. Our choice of $M=7$ in our experiments is justified here when $T=1$.

In Figure \ref{appendix:m_T} we vary the number of federated rounds $T$ alongside the binning parameter $M \in \{4,7,10\}$. Here we observe that when $T > 1$ the choice of $M$ is more important to the overall performance of FedBBQ. We observe that setting $M$ too large $(M=10)$ results in poor test accuracy for large $T$. In Figure \ref{appendix:m_dp} we plot the same experiment but under DP with a privacy budget of $\eps = 3$. Here we observe that while $M=4$ achieves the best balance of test accuracy and cwECE in the non-DP setting (Figure \ref{appendix:m_T}) it does not perform well in the DP setting (Figure \ref{appendix:m_dp}). This further justifies our choice of $M=7$ as, while not optimal across all settings, has consistently good utility (i.e., good balance of cwECE and test accuracy). We note that finding ways to adaptively select this parameter on private data in a federated way remains future work.

\subsubsection{Clipping Norms: FedBBQ}
\label{appendix:bin_clip}
In Figure \ref{fig4_bin} we investigate the performance of binning over a variety of clipping norm choices to understand why it performs poorly compared to scaling in the user-level DP setting. We plot the average test accuracy and cwECE after calibration for clipping norms of the form $[C_+, C_-]$ where $C_+$ is the clipping norm for class positive histograms and $C_-$ is the clipping norm for negative histograms. We color points by their respective positive clipping norm $C_+$. We clearly find the clipping norm has a significant impact on the resulting cwECE and test accuracy of the calibrator, and that the biggest impact is from the choice of norm for positive class histograms. This should be chosen relatively small to maintain both the low cwECE and high test accuracy.

\subsubsection{Clipping Norms: FedTemp with Temperature Clipping}
\label{appendix:temp_clip}

We noted in Table \ref{tab:1} and Table \ref{tab:2} that FedTemp with DP often outperformed the non-DP variant, particularly on CIFAR10. In Figure \ref{fig:temp_clip} we plot FedTemp on CIFAR10 (Simple CNN, $\beta = 0.1$) varying the federated calibration rounds $T$ and the clipping norm $(C)$ where $C=-1$ is no clipping. When $5 < T < 30$ we find clipping the clients temperature parameters before performing FedAvg helps improve the cwECE. This is because it prevents a single client from changing the global temperature parameter too much by their local skew. We note that clipping is not effective in a one-shot setting ($T=1$) or when $T$ is large. This is because when $T$ is large, the effects of local-skew are averaged out over enough steps. When $T=1$, the initial temperature parameter is likely far from the global optimal and so clipping will actually worsen results. %

\subsection{Further Experiments}\label{appendix:experiments}

\begin{figure*}[t!]
\centering
    \begin{center}
        \includegraphics[width=0.9\linewidth]{figures/legend.png} \vspace{-4mm}
    \end{center}
  \subfloat[Varying $T$ -- cwECE \label{appendix:cifar100_ece}]{%
    \includegraphics[width=0.3\linewidth]{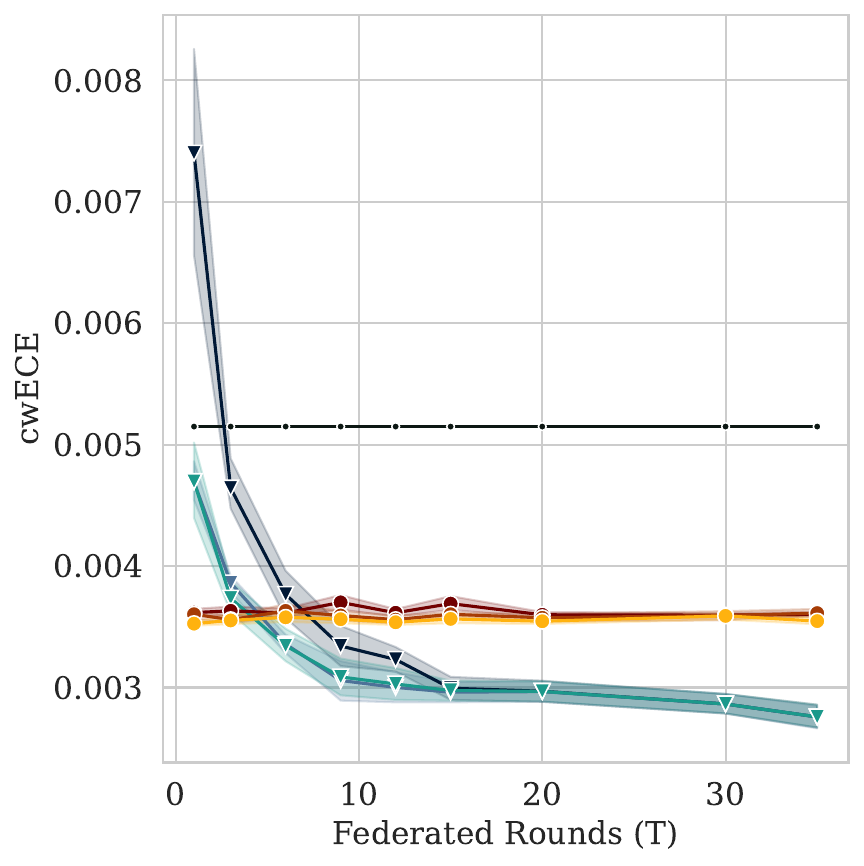}}
  \subfloat[Varying $T$ -- Test Accuracy \label{appendix:cifar100_acc}]{%
        \includegraphics[width=0.3\linewidth]{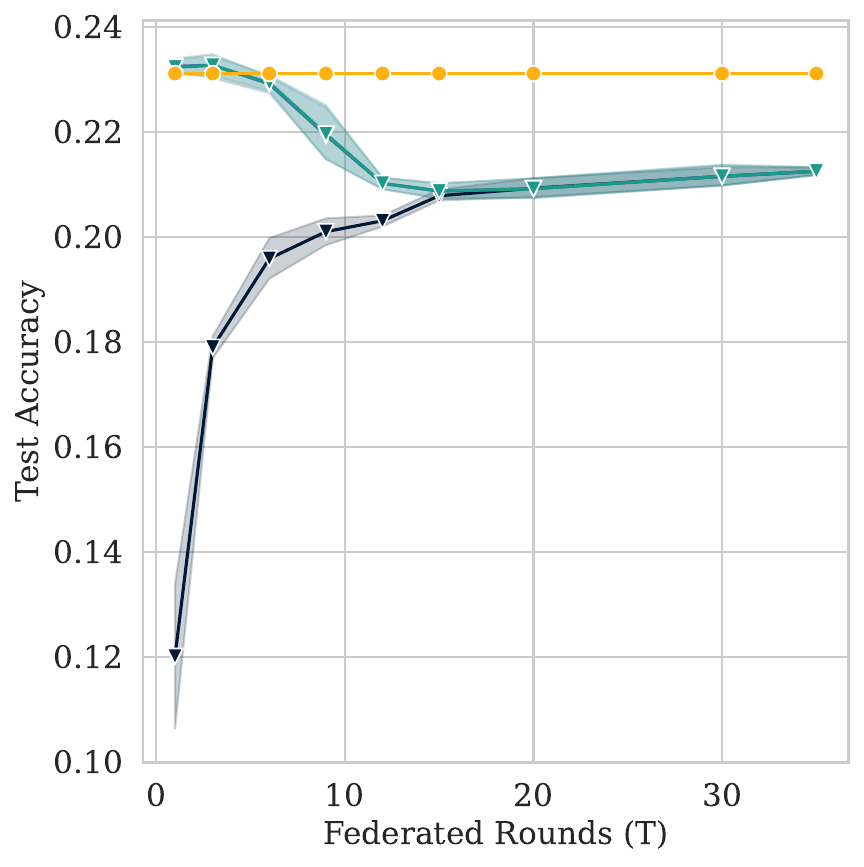}}
  \subfloat[Varying heterogeneity ($\beta$) \label{appendix:cifar100_hetero}]{%
       \includegraphics[width=0.3\linewidth]{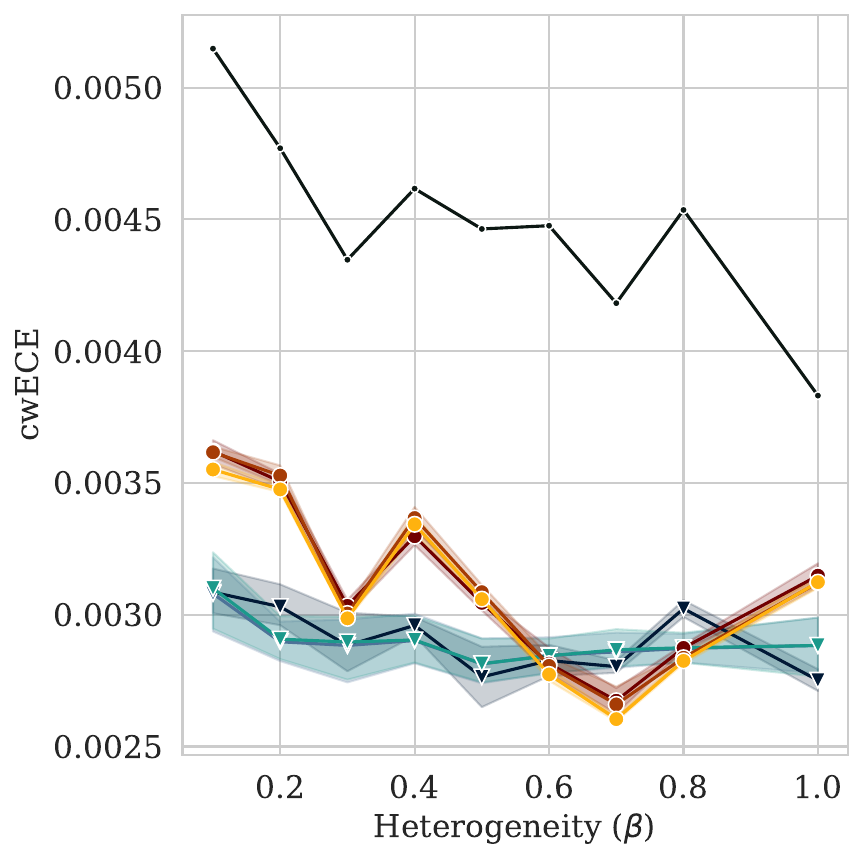}} \\
  \caption{FL Calibration on CIFAR100 (Simple CNN), $\beta=0.1$ unless otherwise stated. \label{appendix:cifar100}}
\end{figure*}
\begin{figure*}[t!]
\centering
    \begin{center}
        \includegraphics[width=0.9\linewidth]{figures/legend.png} \vspace{-4mm}
    \end{center}
  \subfloat[Varying $T$ -- cwECE \label{appendix:mnist_ece}]{%
    \includegraphics[width=0.3\linewidth]{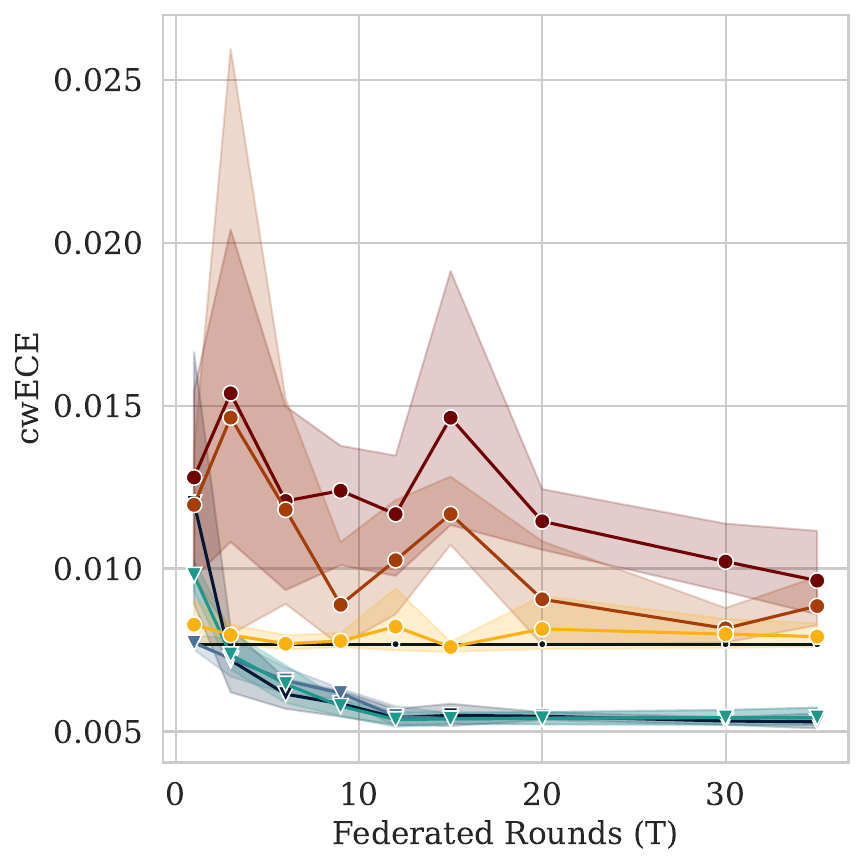}}
  \subfloat[Varying $T$ -- Test Accuracy \label{appendix:mnist_acc}]{%
    \includegraphics[width=0.3\linewidth]{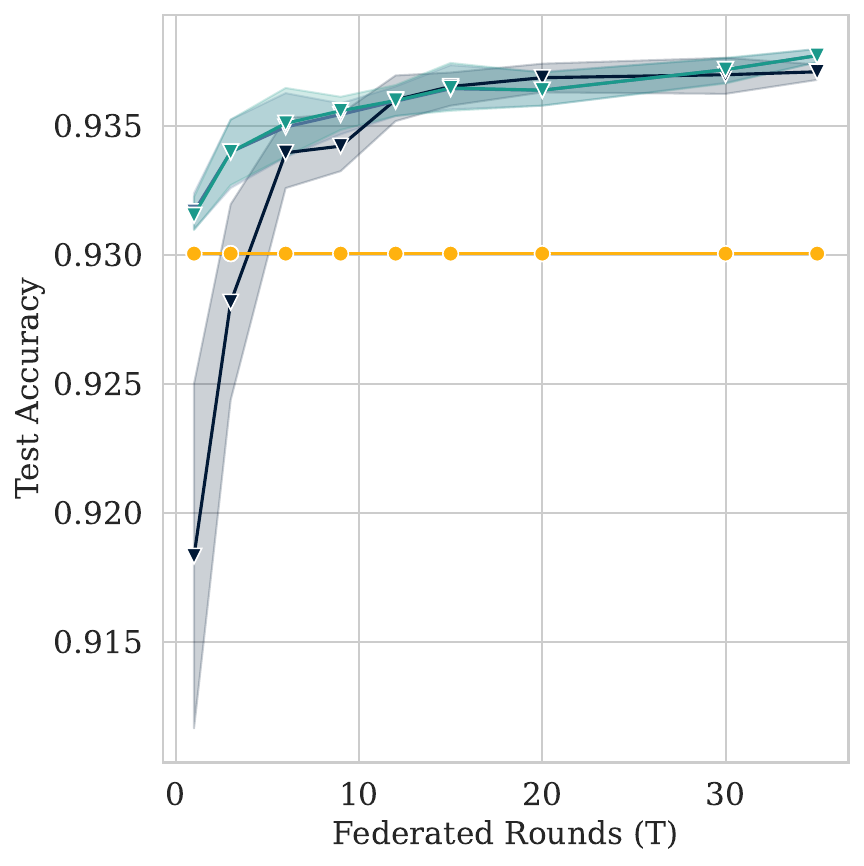}}
  \subfloat[Varying heterogeneity ($\beta$) \label{appendix:mnist_hetero}]{%
       \includegraphics[width=0.3\linewidth]{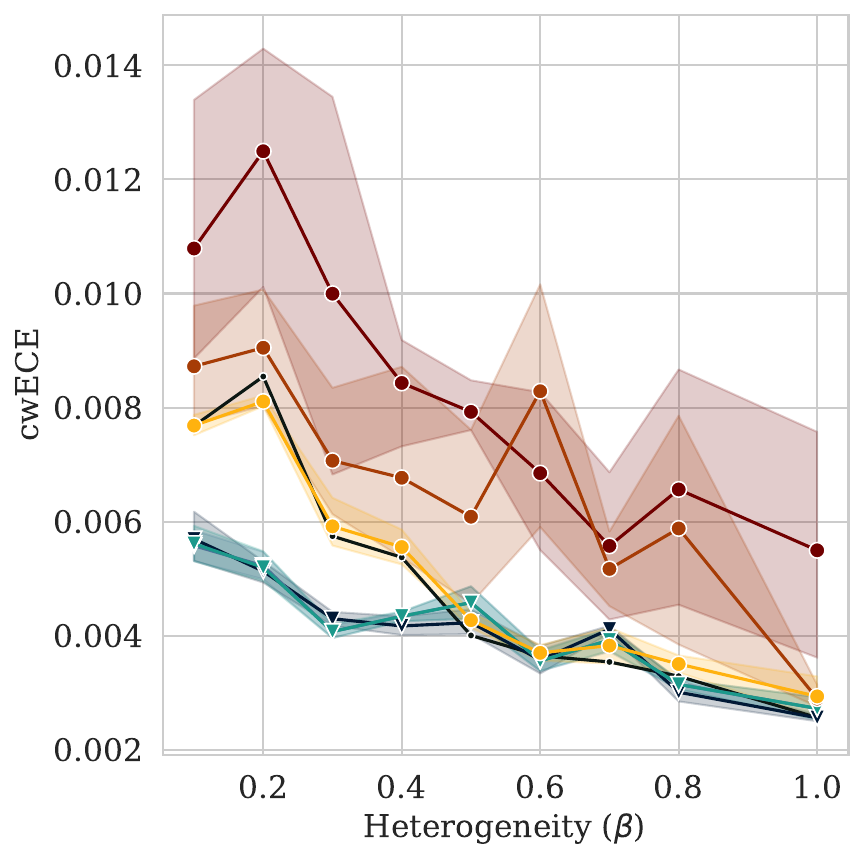}} \\
  \caption{FL Calibration on MNIST (Simple CNN),  $\beta=0.1$ unless otherwise stated. \label{appendix:mnist}}
\end{figure*}

\begin{figure*}[t!]
\centering
    \begin{center}
        \includegraphics[width=0.9\linewidth]{figures/legend.png} \vspace{-4mm}
    \end{center}
  \subfloat[FEMNIST -- cwECE \label{appendix:femnist_ece}]{%
    \includegraphics[width=0.3\linewidth]{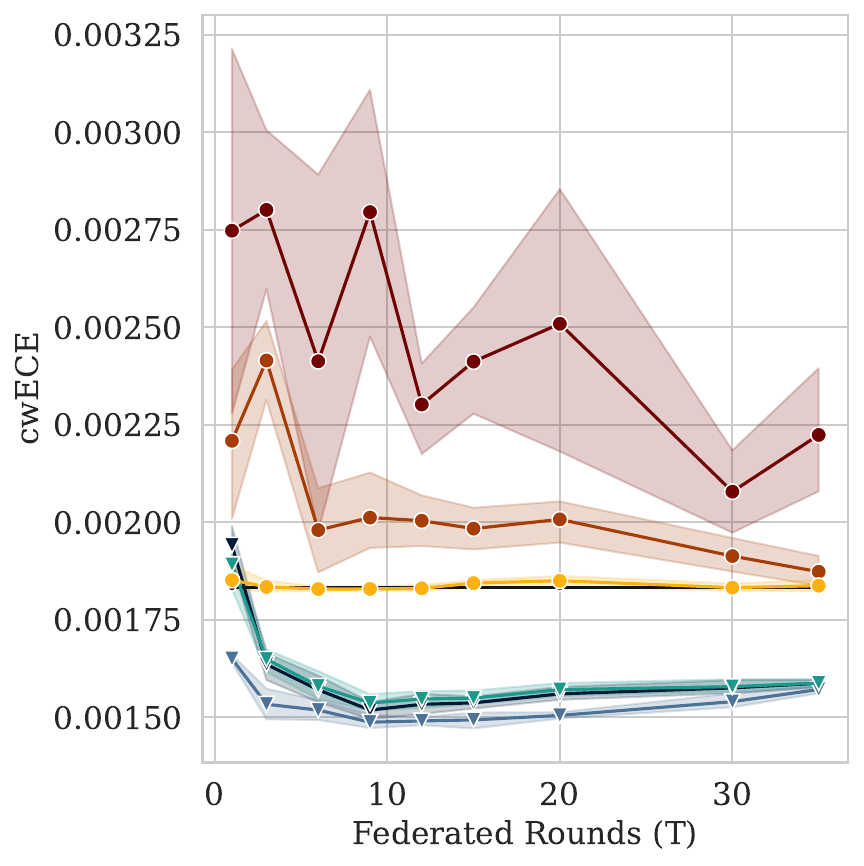}}
  \subfloat[FEMNIST -- Test Accuracy \label{appendix:femnist_acc}]{%
        \includegraphics[width=0.3\linewidth]{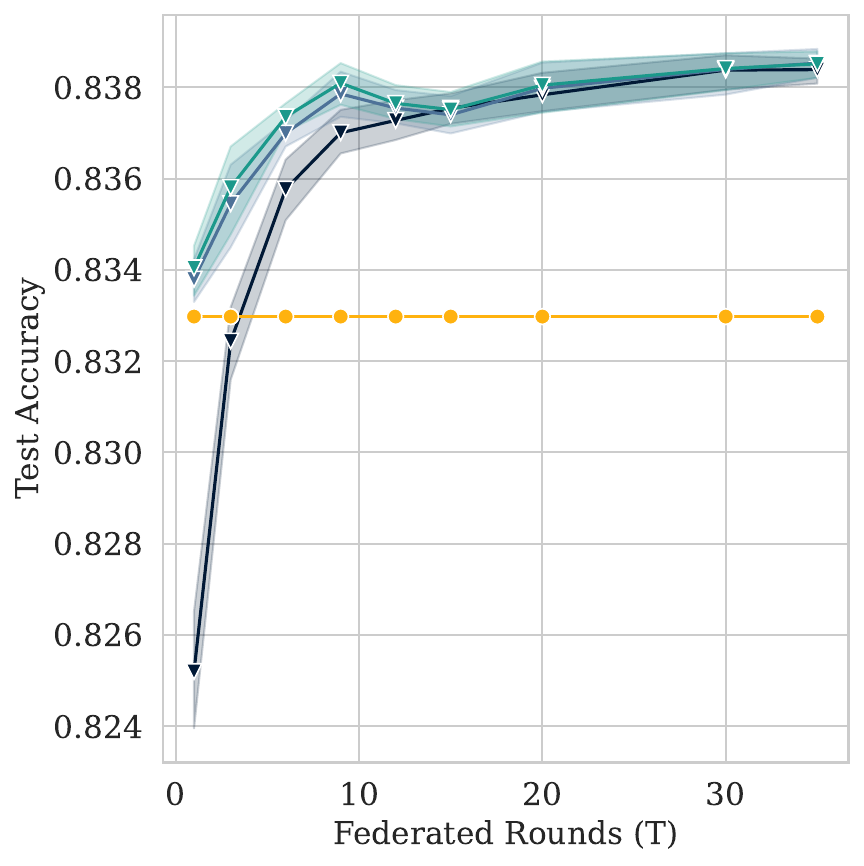}} \\
  \subfloat[Shakespeare -- cwECE \label{appendix:shakes_ece}]{%
        \includegraphics[width=0.3\linewidth]{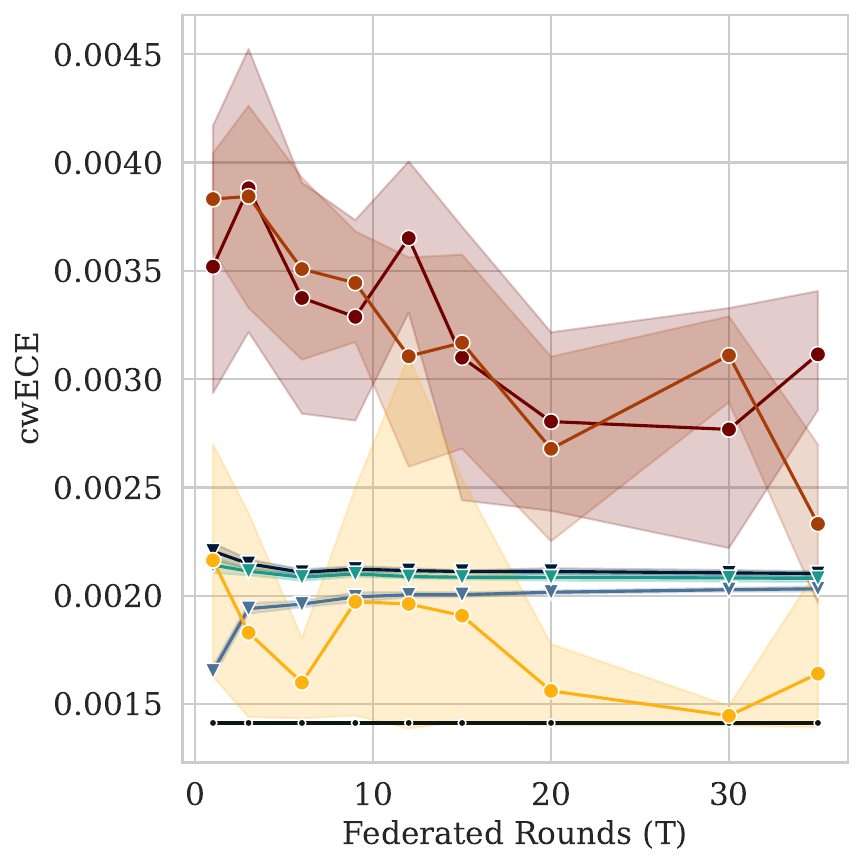}}
  \subfloat[Shakespeare -- Test Accuracy \label{appendix:shakes_acc}]{%
        \includegraphics[width=0.3\linewidth]{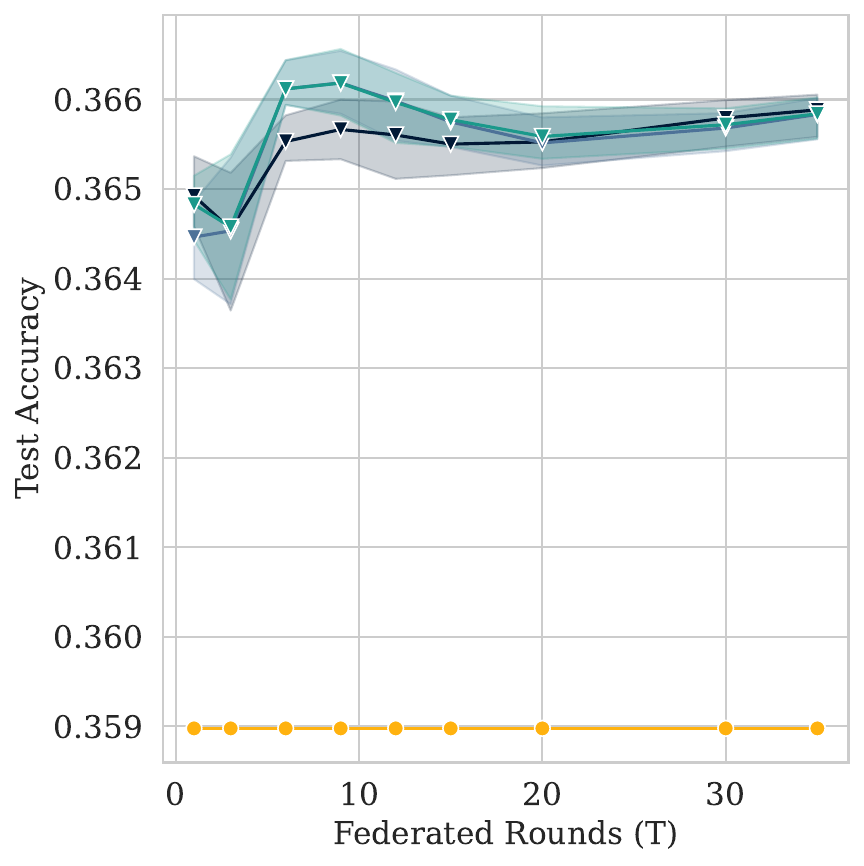}}
  \caption{FL Calibration on LEAF -- Varying $T$ \label{appendix:leaf}}
\end{figure*}

\begin{figure*}[t!]
\centering
    \begin{center}
        \includegraphics[width=0.9\linewidth]{figures/legend.png} \vspace{-4mm}
    \end{center}
  \subfloat[cwECE \label{appendix:cifardp_ece}]{%
    \includegraphics[width=0.3\linewidth]{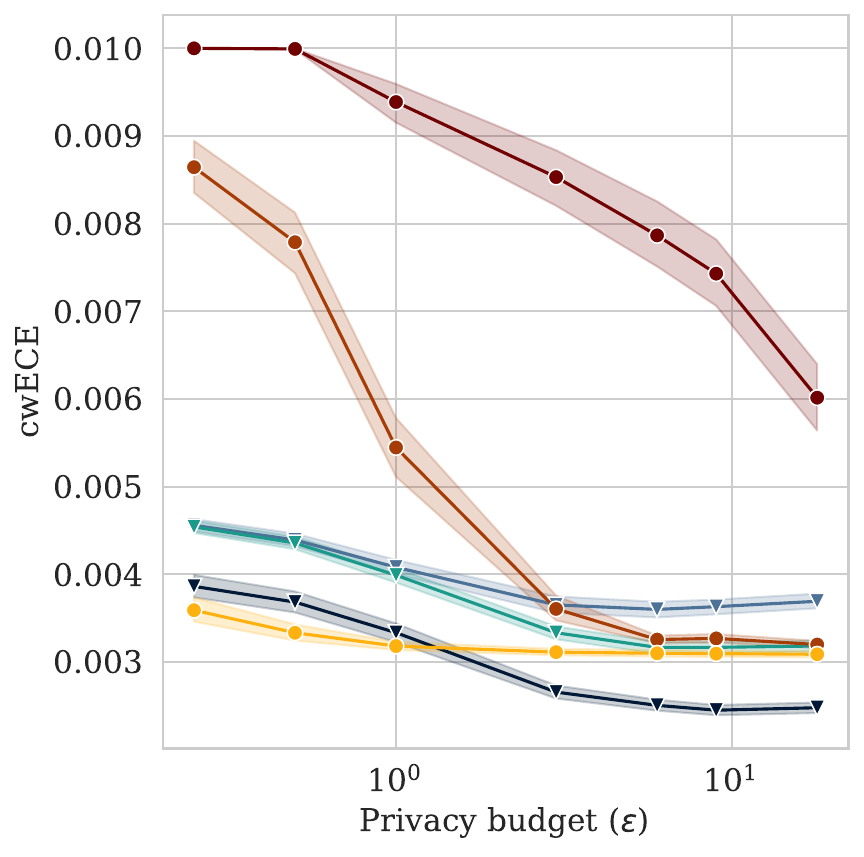}}
  \subfloat[Test Accuracy \label{appendix:cifardp_acc}]{%
        \includegraphics[width=0.3\linewidth]{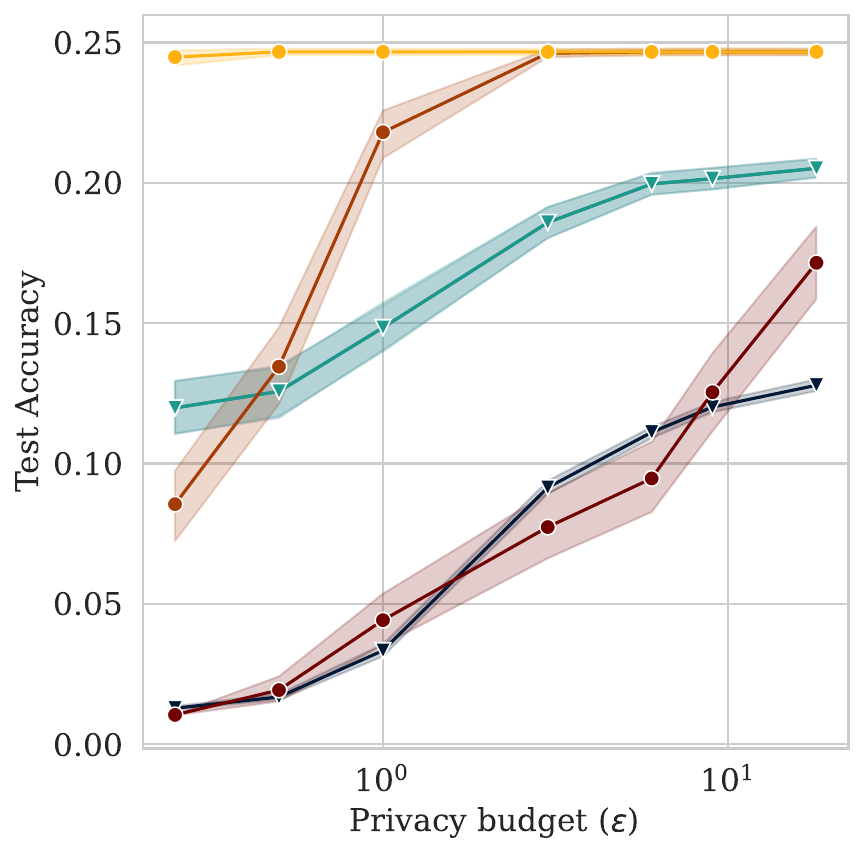}}
  \subfloat[Binning Clipping Norms\label{fig4_bin}]{%
       \includegraphics[width=0.3\linewidth]{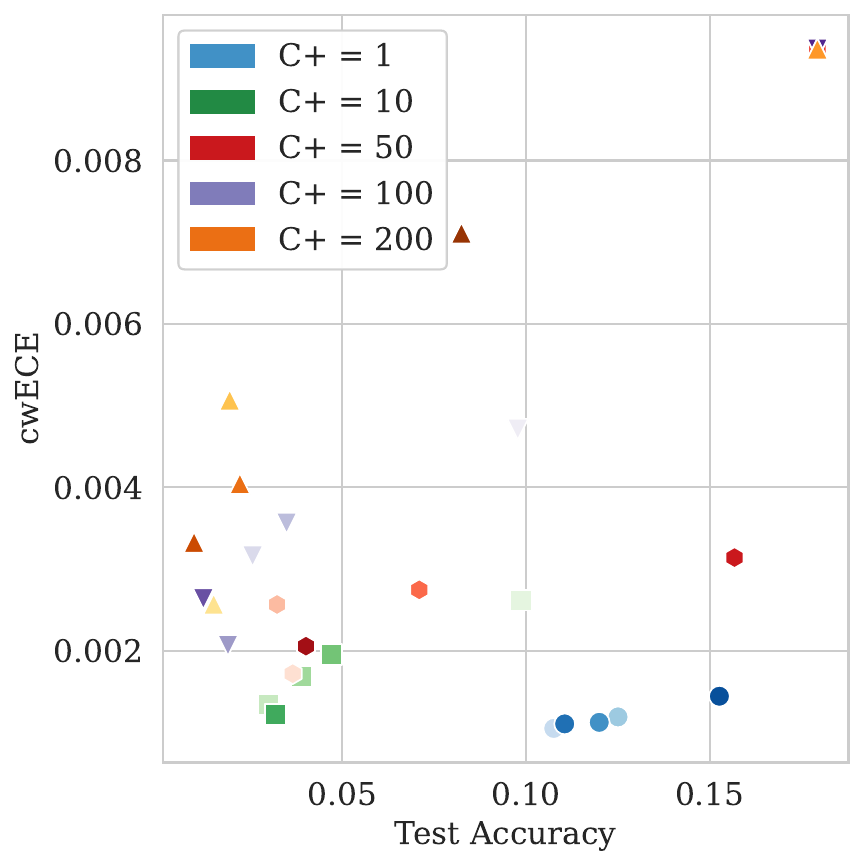}} 
  \caption{DP-FL Calibration - Varying $\eps$ on CIFAR100 (Simple CNN, $\beta = 0.1$) \label{appendix:dpcifar}}
\end{figure*}

\subsubsection{Federated Calibration}
In Figure \ref{appendix:cifar100} and \ref{appendix:mnist} we present FL calibration plots for CIFAR100 and MNIST respectively. In this setting we vary the federated rounds $(T)$ and heterogeneity ($\beta)$ and study both the classwise-ECE and test accuracy after federated calibration. We find results consistent with our conclusions in the main paper on CIFAR10. That is, scaling methods perform well in calibration over a few rounds, but the binning methods achieve best cwECE after training for multiple rounds ($T \geq 30$) as evident in Figure \ref{appendix:cifar100_ece} and Figure \ref{appendix:mnist_ece}. Similarly, when we vary heterogeneity ($\beta$) we observe that when $\beta$ is small (high heterogeneity) it results in higher cwECE than in the IID setting. This trend is more apparent on MNIST (Figure \ref{appendix:mnist_hetero}) than on CIFAR100 where achievable cwECE is fairly constant as heterogeneity is varied (Figure \ref{appendix:cifar100_hetero}). This is likely because the overall model accuracy on CIFAR100 does not change significantly as heterogeneity $(\beta)$ is varied (as seen in Table \ref{tab:models}).

In Figure \ref{appendix:leaf} we present further plots on LEAF datasets. Here we only present experiments varying the number of calibration rounds ($T$) as the LEAF datasets have a fixed partitioning and so cannot vary heterogeneity ($\beta$). For FEMNIST, we observe consistent results as in previous datasets e.g., that binning outperforms scaling over multiple rounds (Figure \ref{appendix:femnist_ece}). For Shakespeare, as seen in the main paper, federated calibration is difficult and no methods improve over the base model.

\begin{figure*}[t!]
\centering
    \begin{center}
        \includegraphics[width=0.3\linewidth]{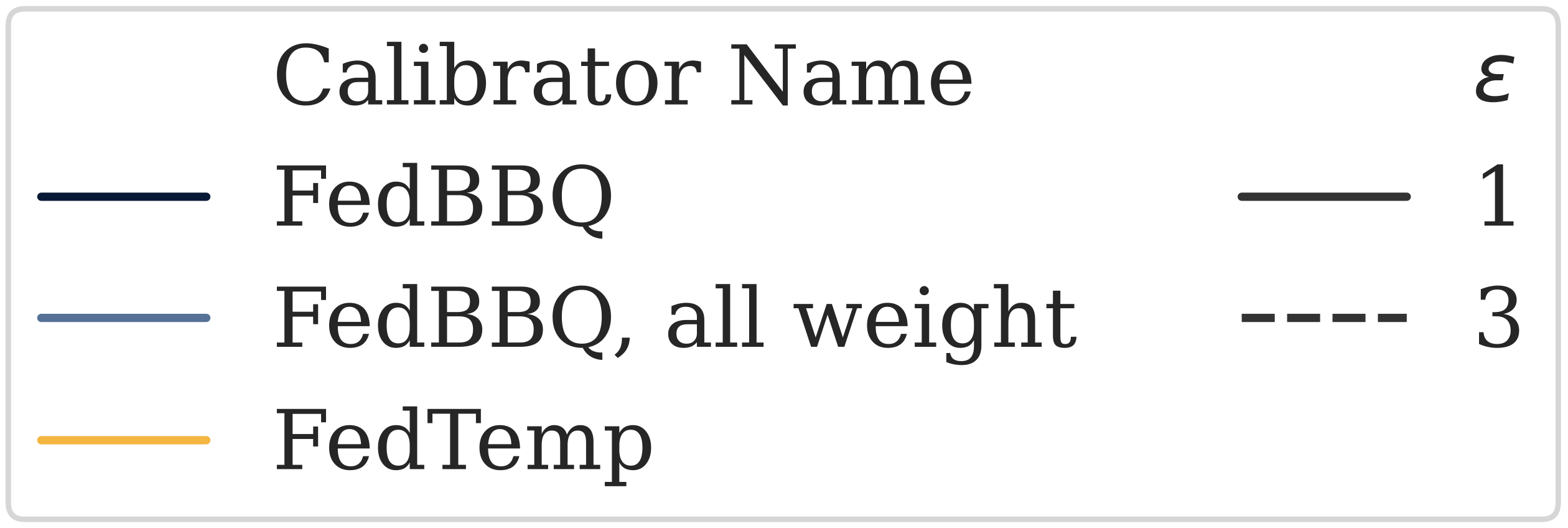} \vspace{-4mm}
    \end{center}
  \subfloat[CIFAR10 -- Test Accuracy\label{appendix:varydp_cifar10_acc}]{%
    \includegraphics[width=0.32\linewidth]{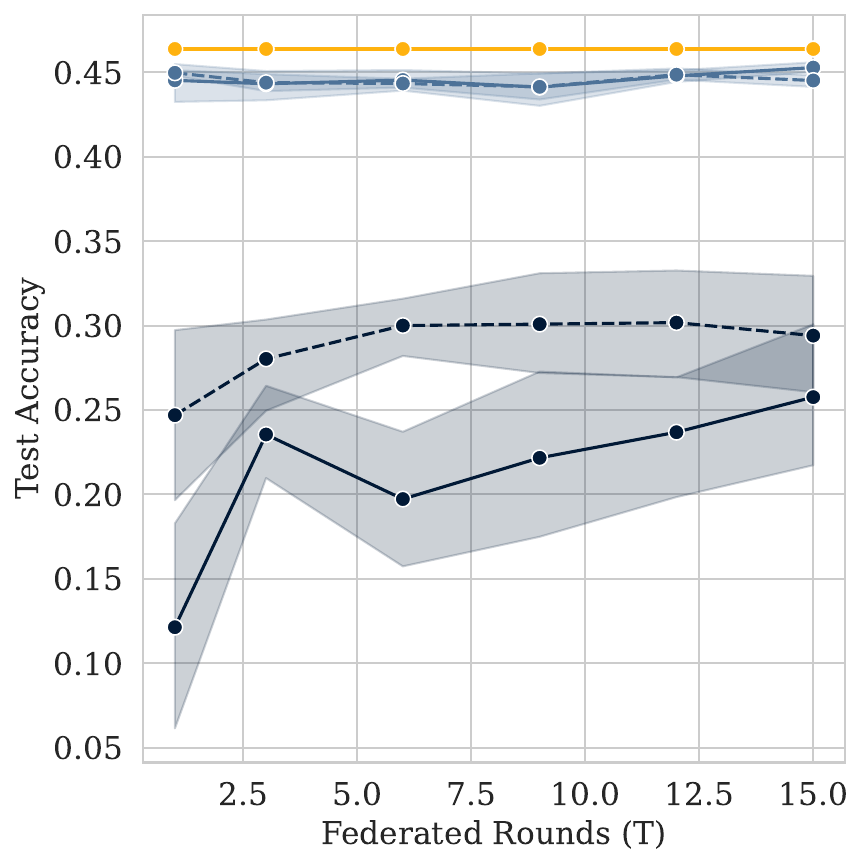}}
  \subfloat[CIFAR10 -- cwECE\label{appendix:varydp_cifar10_ece}]{%
    \includegraphics[width=0.32\linewidth]{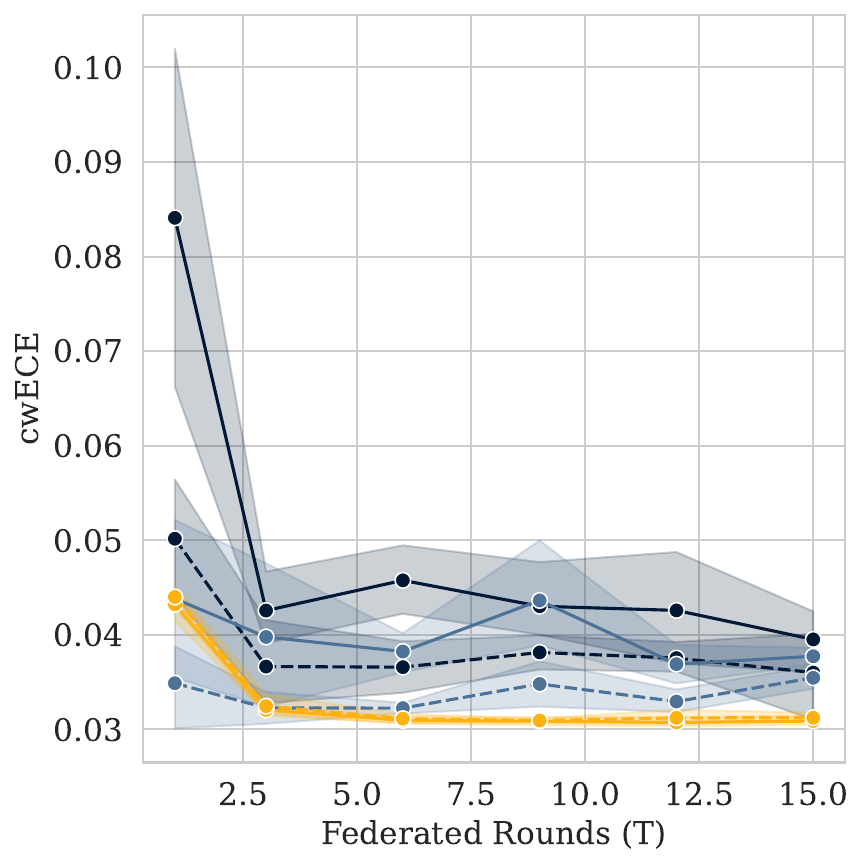}} \\
  \subfloat[CIFAR100 -- Test Accuracy\label{appendix:varydp_cifar100_acc}]{%
        \includegraphics[width=0.32\linewidth]{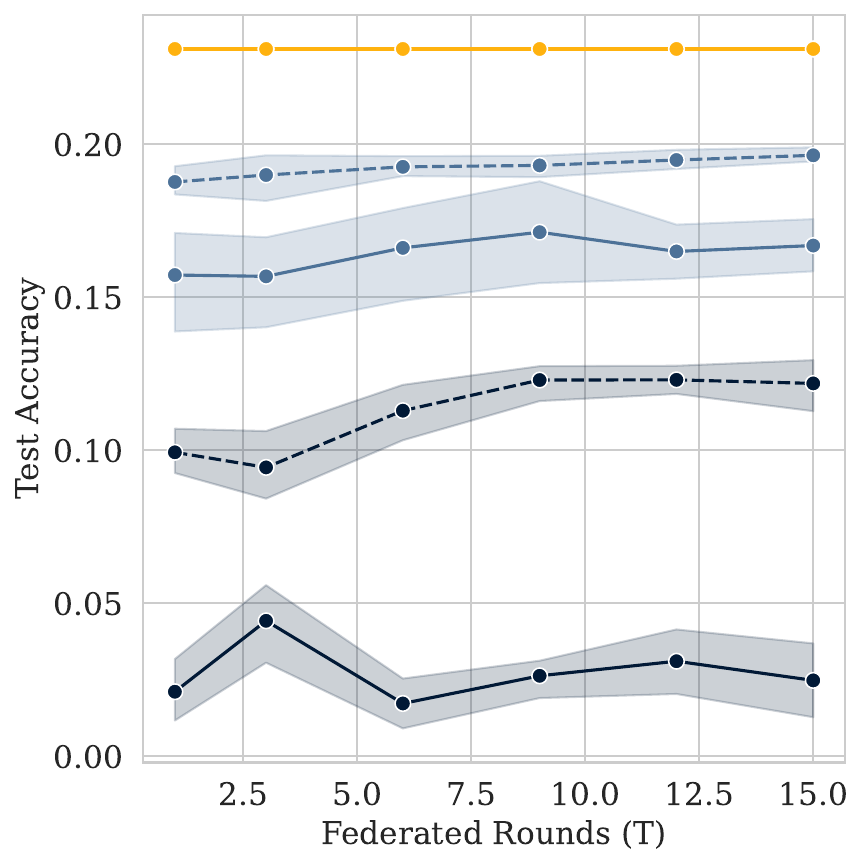}}
  \subfloat[CIFAR100 -- cwECE\label{appendix:varydp_cifar100_ece}]{%
    \includegraphics[width=0.32\linewidth]{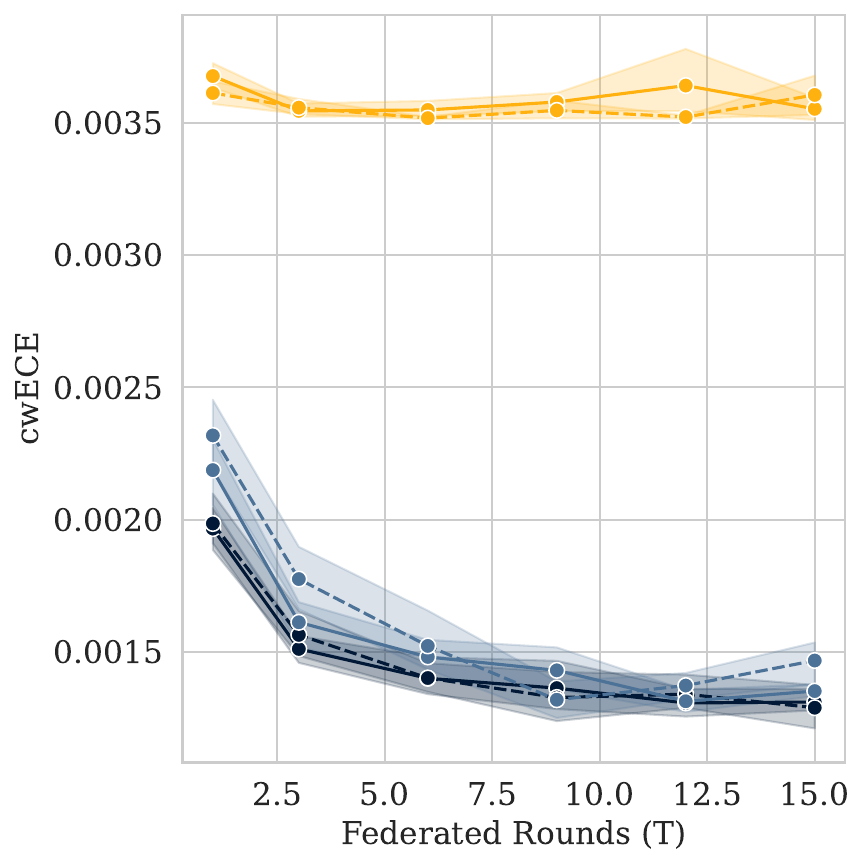}} \\
  \subfloat[FEMNIST -- Test Accuracy\label{appendix:varydp_femnist_acc}]{%
        \includegraphics[width=0.32\linewidth]{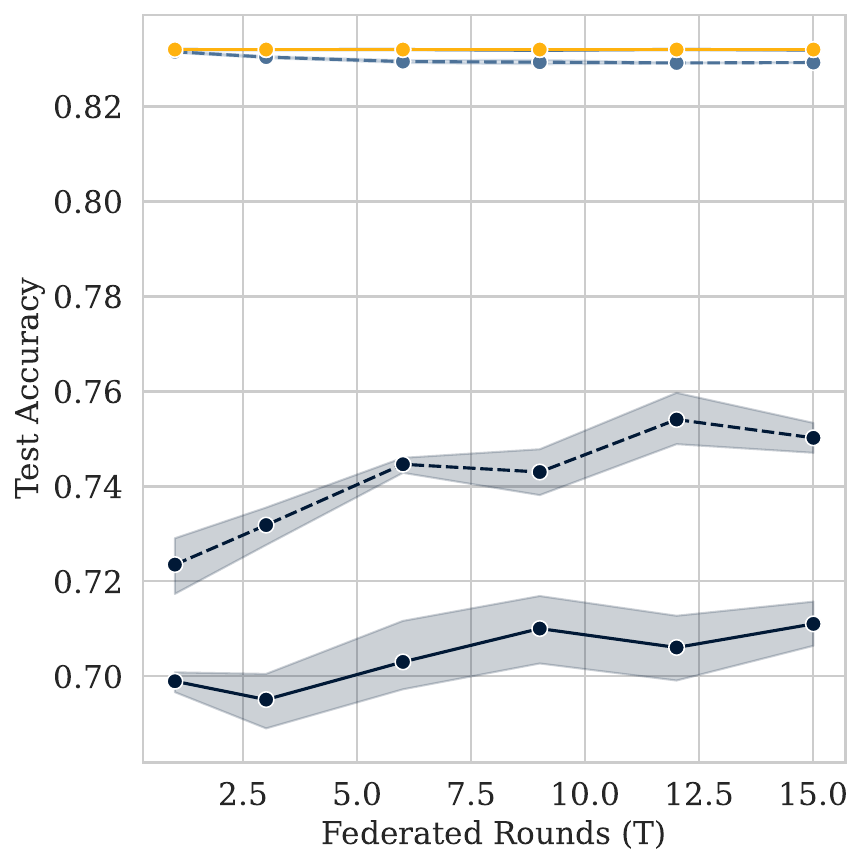}}
  \subfloat[FEMNIST -- cwECE\label{appendix:varydp_femnist_ece}]{%
    \includegraphics[width=0.32\linewidth]{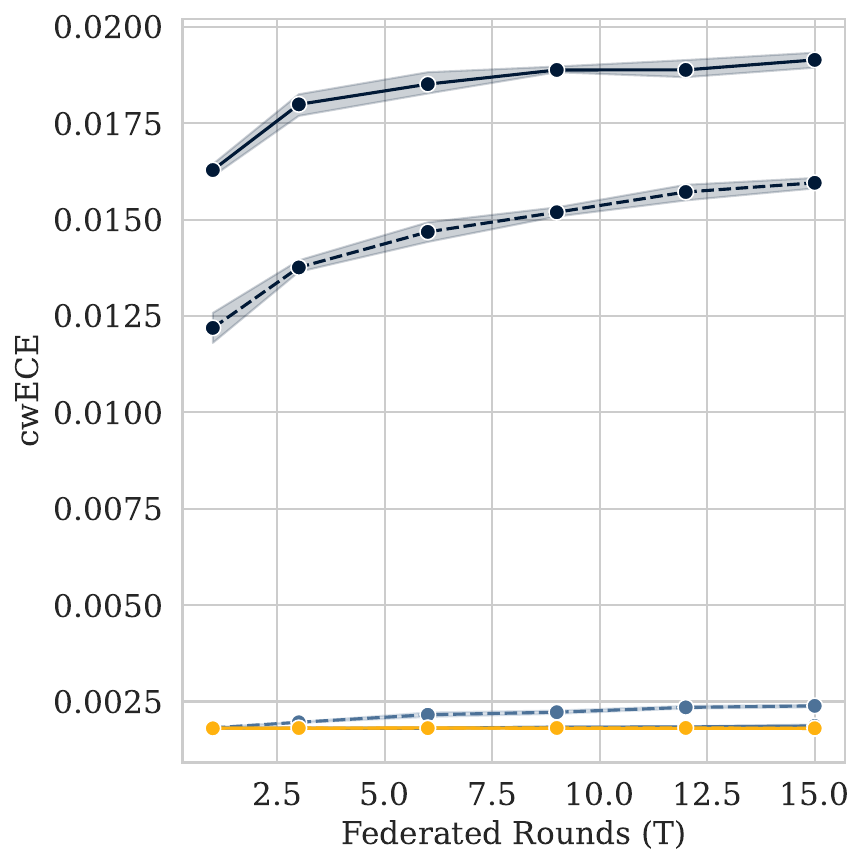}} \\
  \caption{DP-FL Calibration, Varying $T$ and $\eps$, $\beta=0.1$ except FEMNIST which uses LEAF. \label{appendix:dpvary}}
\end{figure*}

\subsubsection{DP-FL Calibration}\label{appendix:dp_fl}

\begin{figure*}[t!]
\centering
  \subfloat[CIFAR10 (no DP): cwECE\label{fig:nodp_cwece}]{%
        \includegraphics[width=0.32\linewidth]{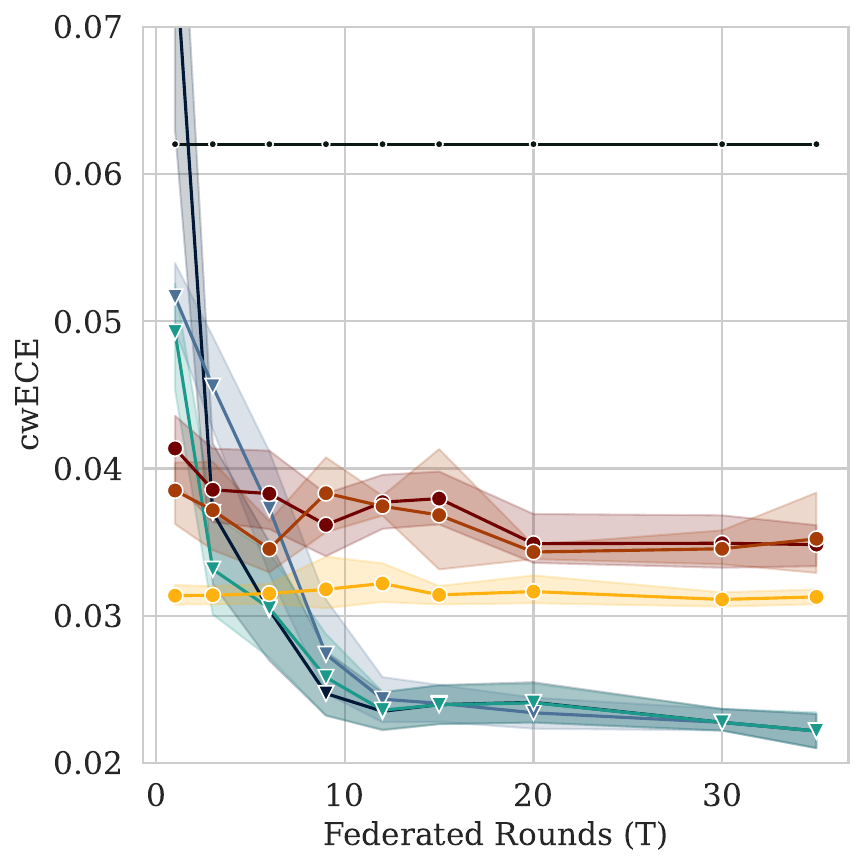}}
  \subfloat[CIFAR10 (no DP): ECE\label{fig:nodp_ece}]{%
    \includegraphics[width=0.32\linewidth]{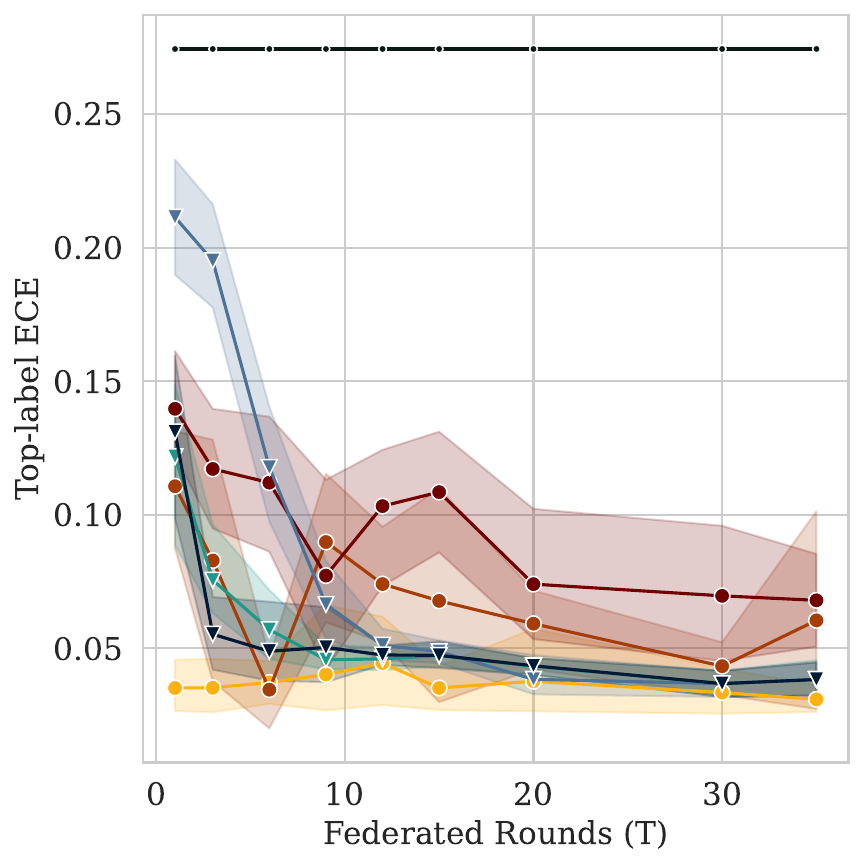}} \\
  \subfloat[CIFAR10 (DP): cwECE\label{fig:dp_cwece}]{%
    \includegraphics[width=0.32\linewidth]{figures/fig3a_cifar10_simple_classwise_ece.pdf}}
  \subfloat[CIFAR10 (DP): ECE\label{fig:dp_ece}]{%
    \includegraphics[width=0.32\linewidth]{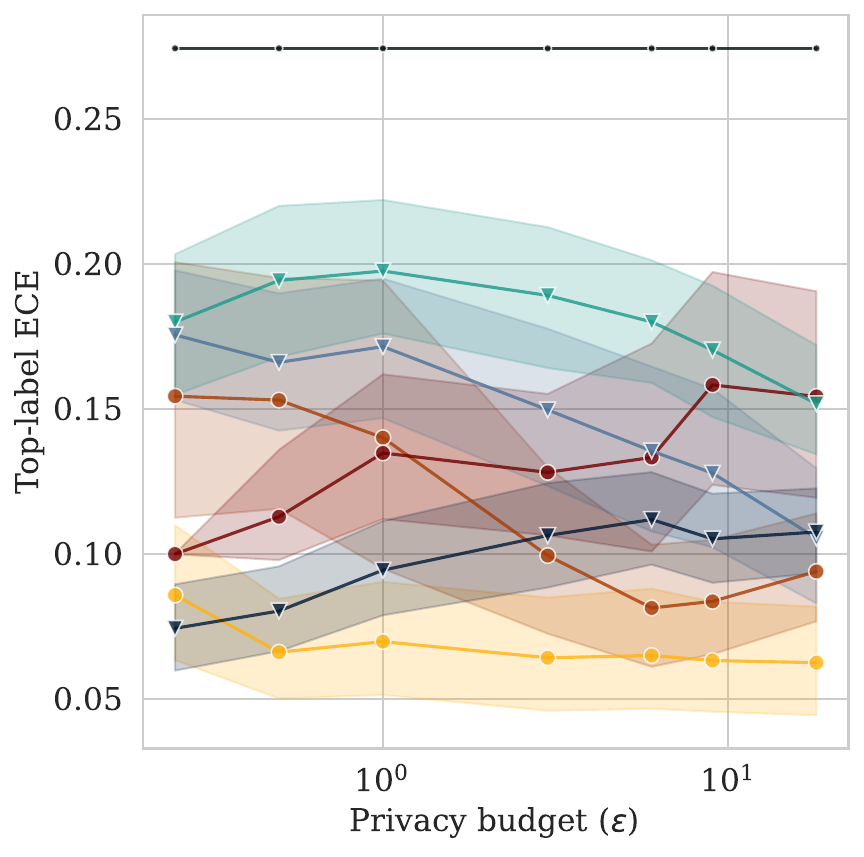}} 
    \caption{cwECE vs. ECE on CIFAR10 (Simple CNN, $\beta = 0.1$)\label{fig:ece}}
\end{figure*}

\paragraph{Replication of Figure~\ref{fig:4}} In Figure \ref{appendix:dpcifar} we present DP-FL calibration results for CIFAR100. Here we vary $\eps$ whilst studying both cwECE and test accuracy. Results here are consistent with our conclusions made in Section \ref{sec:exp} for DP-FL on CIFAR10, which is that, under strict DP  (small $\eps$), FedTemp performs best. However, CIFAR100 shows a clearer trend that as $\eps$ grows large, the FedBBQ methods perform best, highlighting that FedBBQ only works well in DP settings with little noise.

\paragraph{Varying $T$ and $\eps$.} In Figure \ref{appendix:dpvary} we vary both the number of rounds ($T$) and the privacy budget ($\eps$) for CIFAR10, CIFAR100 and FEMNIST. We study $\eps=1,3$. We clearly observe that our weighting method helps preserve the test accuracy of binning calibrators, especially under strict DP ($\eps=1$). For example, on CIFAR10 (Figure \ref{appendix:varydp_cifar10_acc}), the only method to significantly lose test accuracy after federated calibration are versions of FedBBQ without any weighting scheme. This is consistent on CIFAR100 and FEMNIST. For cwECE, we find FedTemp has consistent performance as both $T$ and $\eps$ is varied. However, as discussed in Section \ref{sec:exp}, as the binning methods are more sensitive to noise, their performance improves (often significantly) as both $T$ and $\eps$ increases. %

\begin{figure*}[t!]
\centering
    \begin{center}
        \includegraphics[width=0.4\linewidth]{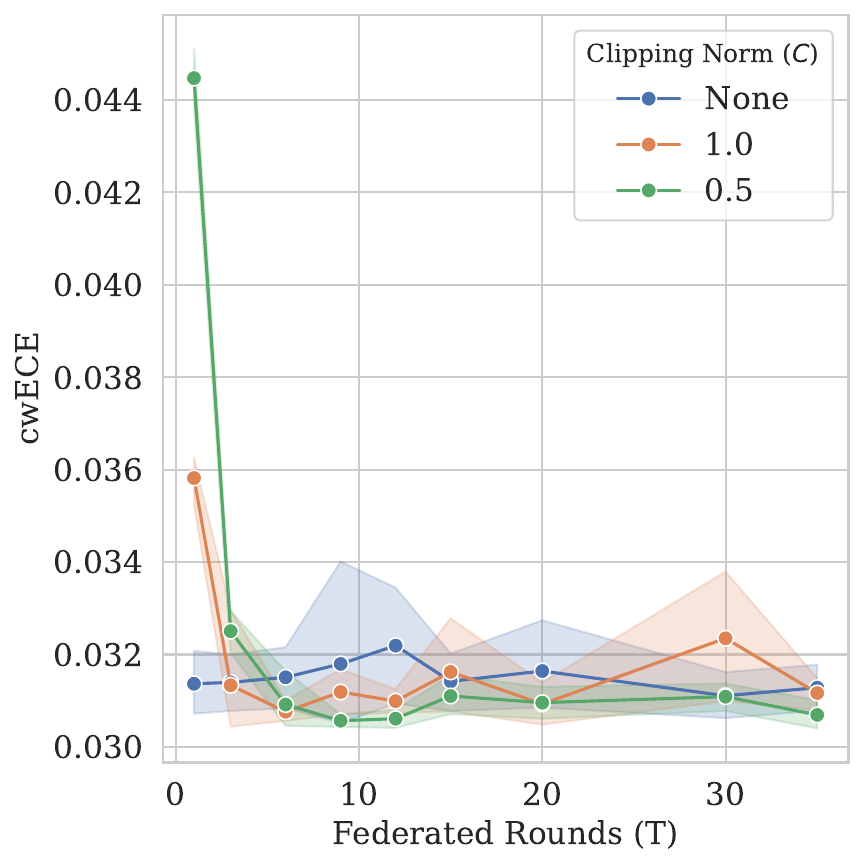}
    \end{center}
  \caption{FedTemp on CIFAR10 (Simple CNN, $\beta = 0.1$) - Varying clipping norm $C$\label{fig:temp_clip}}
\end{figure*}

\subsection{Top-label ECE vs. Classwise ECE}
\label{appendix:ece}

In Figure \ref{fig:ece}, we present the cwECE results shown for CIFAR10 (Simple CNN, $\beta = 0.1$) in the main paper where Figure \ref{fig3_ece} is the same as Figure \ref{fig:nodp_cwece} and Figure \ref{fig4_ece} is the same as \ref{fig:dp_cwece}. We present the equivalent figures for the (top-label) ECE metric. These demonstrate why we focus only on cwECE in our work. First, for the FL (no DP) setting in Figures \ref{fig:nodp_cwece} (cwECE) and \ref{fig:nodp_ece} (ECE) we see that training for multiple rounds results in both binning and scaling methods achieving similar ECE scores yet the cwECE results are very different, showing a gap between scaling and binning approaches. This highlights why we focus on cwECE as it tells a more complete picture for the federated setting. For the DP-FL setting in Figures \ref{fig:dp_cwece} (cwECE) and \ref{fig:dp_ece} (ECE) we find the ECE results are more consistent with the cwECE conclusions e.g., that scaling methods are preferred to binning approaches under DP.

\begin{table*}[t]
    \caption{Simulated communication overhead of local participants who participate in federated calibration on CIFAR10 (ResNet18). \label{tab:comm_cifar10} The parameter $c$ is the total number of classes, $N$ is the total number of hidden neurons in the FedCal MLP scaler and $M$ is the binning parameter in BBQ.}
    \centering
    \small
    \begin{tabular}{lllll}
    \toprule
     Calibrator & Communication Complexity & Send Size & Receive Size & Model Size \\
     & & (per-round) & (per-round) & (received once) \\
    \midrule
    FedTemp & O(1) & 0.01kb & 0.01kb & 43678.29kb \\
    FedOPVector & O(c) & 0.16kb & 0.16kb & 43678.29kb \\
    FedCal & O($N^2 + c^2$) & 43.08kb & 43.08kb & 43678.29kb \\
    FedBBQ & O($2^Mc$) & 20.0kb & 0.0kb & 43678.29kb \\
    FedBBQ, all weight & O($2^Mc$) & 20.0kb & 0.0kb & 43678.29kb \\
    \bottomrule
    \end{tabular}
    \vspace{-4pt}
\end{table*}
\begin{table*}[t]
    \caption{Local client per-round computation time for federated calibration on CIFAR10 (ResNet18). Time is measured in seconds. \label{tab:comp_cifar10}}
    \centering
    \small
    \begin{tabular}{llrrr}
    \toprule
     Calibrator & Client Time (Mean) & Client Time (Min) & Client Time (Max) \\
    \midrule
    FedBin & 0.005s & 0.003s & 0.008s \\
    FedBBQ & 0.032s & 0.030s & 0.036s \\
    FedBBQ, all weight & 0.039s & 0.031s & 0.049s \\
    \hline
    FedTemp & 0.111s & 0.099s & 0.124s \\
    FedOPVector & 0.017s & 0.013s & 0.026s \\
    FedCal & 0.008s & 0.006s & 0.015s \\
    \bottomrule
    \end{tabular}
    \vspace{-4pt}
\end{table*}
 \subsection{Computation and Communication Overhead}
 \label{appendix:overhead}

 In this section we present benchmarks for the computation and communication overhead of the federated calibration methods we consider in this work. 

For communication overhead, we present the per-round communication cost of a client participating in federated calibration across a range of methods on CIFAR100 in Table \ref{tab:comm_cifar10}. We do this in a simulated setting where we do not measure the communication overhead of using secure-aggregation protocols. Instead, these benchmarks are the raw communication cost of sending the calibration model parameters to and from the local clients. Observe the following:
 \begin{itemize}
     \item FedTemp and FedOPVector require sending 1 parameter (the temperature) and $2c$ parameters (the classwise scaling parameters) respectively. This results in very little per-round communication which is dwarfed by the cost of having to receive the federated classifier model from the server.
     \item FedBBQ and our FedBBQ variation with a weighting scheme has identical communication cost which requires sending a histogram of size $2^M$ for each positive and negative class resulting in $2\cdot (2^Mc)$ total communication. This is more costly than scaling approaches but better than FedCal as it does not need to receive any calibration model parameters from the server since FedBBQ just requires local histograms to be computed and sent to the server.
     \item FedCal requires sending a small neural network to and from the clients. We use the same architecture as was proposed in the original paper \citep{peng2024fedcal}. This has 2 hidden layers with 64 neurons but crucially scales with the number of classes $c$ (as the input and outputs of the MLP). This method has the largest overhead as it requires receiving and sending this MLP back to the server at each round.
 \end{itemize}

 In Table \ref{tab:comp_cifar10} we present the average, min and max computation time required by local clients when they participate in a single round of federated calibration. Observe that all methods are lightweight requiring at most $0.1$ seconds of local compute time. We observe that FedBBQ is more costly than FedBin as it requires computing larger histograms but this is still a lightweight procedure. For scaling, we observe that FedCal has the fastest training time compared to temperature and vector scaling. This difference arises because we follow the FedCal paper default setting of only $3$ training epochs trained via SGD, while our implementations of FedTemp and FedOPVector involve solving a convex optimization problem where we set a relatively large number of maximum iterations ($50$). Nevertheless, none of the methods introduce prohibitive computational costs, so this is not a limiting factor for federated calibration. Instead, the key consideration lies in selecting methods that achieve best calibration error whilst minimizing communication overhead as discussed above.
 
\section{Algorithm Details}
\subsection{Differential Privacy}\label{appendix:dp}

In this section, we present and prove the privacy guarantees of our federated calibration approaches under user-level DP. 
For completeness, we provide additional definitions and results, starting with the definition of $(\eps,\delta)$-Differential Privacy.
\begin{definition}[Differential Privacy~\citep{dwork2014foundations}]
A randomised algorithm $\mathcal{M}\colon \mathcal{D} \rightarrow \mathcal{R}$ satisfies $(\varepsilon, \delta)$-differential privacy if for any two adjacent datasets $D, D^\prime \in \mathcal{D}$ and any subset of outputs $S \subseteq \mathcal{R}$, $$\prob(\mathcal{M}(D) \in S) \leq e^\eps \prob(\mathcal{M}(D^\prime) \in S) + \delta.$$
\end{definition}
We will first provide guarantees with the more convenient $\rho$-zCDP formulation.
\begin{definition}[$\rho$-zCDP] \label{def:zcdp}
A mechanism $\mathcal{M}$ is $\rho$-zCDP if for any two neighbouring datasets $D, D^\prime$ and all $\alpha \in (1,\infty)$ we have $D_\alpha(\mathcal{M}(D) | \mathcal{M}(D^\prime) \leq \rho \cdot \alpha$,
where $D_\alpha$ is Renyi divergence of order $\alpha$.
\end{definition}

It is common to then translate this guarantee to the more interpretable $(\eps,\delta)$-DP via the following lemma.

\begin{lemma}[zCDP to DP \citep{canonne2020discrete}]
\label{lemma:cdp}
If a mechanism $\mathcal{M}$ satisfies $\rho$-zCDP then it satisfies $(\eps,\delta)$-DP for all $\eps > 0$ with $$\delta = \min_{\alpha > 1} \frac{\exp((\alpha-1)(\alpha\rho - \eps))}{\alpha - 1} \left(1-\frac{1}{\alpha}\right)^\alpha$$
\end{lemma}

In this work we are concerned with \emph{user-level} differential privacy. In order to provide user-level DP guarantees we must bound the contribution of any one user. The following notion of (user-level) sensitivity captures this.

\begin{definition}[$L_2$ Sensitivity]
Let $f: \mathcal{D} \rightarrow \R^d$ be a function over a dataset. The $L_2$ sensitivity of $f$, denoted $\Delta_2(f)$, is defined as $\Delta_2(f) := \max_{D \sim D^\prime} \|f(D) - f(D^\prime)\|_2$, where $D \sim D^\prime$ represents the user-level relation between datasets i.e., $D^\prime$ is formed from the addition or removal of an entire user's data from $D$. 
\end{definition}

A standard mechanism for guaranteeing differential privacy guarantees on the numerical outputs of an algorithm is through the use of the Gaussian mechanism.

\begin{definition}[Gaussian Mechanism, GM]\label{def:gm}
Let $f: \mathcal{D} \rightarrow \R^d$, the Gaussian mechanism is defined as $GM(f) = f(D) + \Delta_2(f) \cdot \mathcal{N}(0, \sigma^2 I_d)$.
The GM satisfies $\frac{1}{2\sigma^2}$-zCDP.
\end{definition}

As the Gaussian mechanism is invoked at each federated round, we would also like to compose these privacy guarantees.

\begin{lemma}[zCDP composition \citep{bun2016concentrated}]
\label{lemma:comp}
If a mechanism $\mathcal{M}$ satisfies $\rho_1$-zCDP and mechanism $\mathcal{M}^\prime$ satisfies $\rho_2$-zCDP, then the composition $M^*(x) := M^\prime(M(x))$ satisfies $(\rho_1 +\rho_2)$-zCDP
\end{lemma}

We assume a trusted honest-but-curious server adds Gaussian noise to the quantities that have been aggregated via secure aggregation as outlined in the threat model at the end of Section \ref{sec:fl}. For scaling, noise is added to the (clipped) parameters of the scaling model whilst for binning noise is added to the (clipped) histograms. For both scaling and binning, the training of the calibrator is thus the composition of the Gaussian mechanism over multiple federated rounds. We formalise this in Lemma~\ref{lemma:final}.

\begin{lemma}[Noise calibration for $(\eps, \delta)$-DP federated calibration]
\label{lemma:final}
For any number of federated calibration rounds $T$, federated scaling and binning approaches satisfy $(\eps, \delta)$-DP, by computing $\rho$ according to Lemma \ref{lemma:cdp} and setting 
\begin{align*}
    \sigma &= \begin{cases}
      C\sqrt{\frac{T}{2\cdot \cdot \rho}}, & \text{Scaling} \\
      C\sqrt{\frac{2cT}{2\cdot \rho}}, & \text{Binning}
    \end{cases}
\end{align*}
\end{lemma}

\begin{proof} 
As the server adds noise to the quantities received from secure-aggregation, it suffices to apply Lemma \ref{lemma:comp} with Definition \ref{def:gm}. To find $\sigma$, we simply need to count the number of times the Gaussian mechanism is used. For scaling approaches, we train via DP-FedAvg which has clients send a clipped update with norm $C$ at each round. This is performed over $T$ federated rounds and hence the noise scales proportional to $T$.
For binning, the privacy guarantees are the same as above except each user sends two (clipped) histograms over their datasets, one for positive and one for negative examples. For the case of multi-class classification with $c > 2$, this happens for each class. Hence users send $2c$ histograms per federated round and so the noise scales proportional to $2cT$. 
\end{proof}

For settings where the client participation rate $p < 1$ we can use subsampling amplification via Renyi Differential Privacy (RDP) to get improved composition results, see \cite{mironov2017renyi} for more information.

\subsection{Federated Averaging}\label{appendix:fedavg}

We train all of our models and our scaling calibrators via variations of (DP)-FedAvg. At step $t$ of FedAvg training, we subsample clients from the population with probability $p$ and have them train their local model with current global weights $\vw^t$ via local SGD for a number of epochs, producing local weights $\vw_k^t$. Clients calculate a model update of the form
$ u_k := \vw^t - \vw_k^t$ which acts a pseudo-gradient. The server updates the global model of the form 
$\vw^{t+1} := \vw^{t} - \eta_S\cdot\frac{1}{p\cdot K}\sum_k (\vw^t - \vw_k^{t}) = \vw_k^{t} - \eta_S\cdot\frac{1}{p\cdot K}\sum_k u_k$, where $\eta_S$ is a server learning rate. In the context of DP-FedAvg, the individual model updates $u_k$ are clipped to have norm $C$ and aggregated by the server via secure-aggregation. Denoting $\bar{u}_k = \operatorname{clip}(u_k, C)$ the server computes the final noisy update of the form $\tilde u := \sum_k \bar{u}_k + N(0, C\sigma^2I_d)$

\subsection{Histogram Binning\label{appendix:binning}}

One of the simplest calibration methods is histogram binning \citep{zadrozny2002transforming}. In a binary classification setting, given the output confidence $\hat p_i \in [0,1]$, class-label $y_i \in \{0,1\}$ and a fixed total number of bins $M$, the output confidences of the model are partitioned using fixed-width bins of the form $B_m = [\frac{m-1}{M}, \frac{m}{M}]$ for $m \in \{1, \cdots, M\}$. The calibrator $g(\hat p)$ is then formed as 
$$g(\hat p) = \sum_m \mathbf{1}\{\hat p_i \in B_m\} \frac{P(m)}{P(m) + N(m)}$$

Where the entries of the class-positive histogram $P$ and class-negative histogram $N$ is defined as
$$P(m) := \sum_i \mathbf{1}\{y_i = 1  \wedge \hat p_i \in B_m\}$$
$$N(m) := \sum_i \mathbf{1}\{y_i = 0\ \wedge \hat p_i \in B_m\} $$

To extend this to a multiclass setting, we can learn $c$ one-vs-all calibrator models $g_1, \cdots g_c$ for each class. For a particular example $i$ and class $j$, with predicted class-confidence $\hat p_{i,j}$, the class-positive and negative histograms $P_j(m)$ and $N_j(m)$ are computed as
$$P_j(m) := \sum_i \mathbf{1}\{y_i = j  \wedge \hat p_{i,j} \in B_m\}$$
$$N_j(m) := \sum_i \mathbf{1}\{y_i \neq j\ \wedge \hat p_{i,j} \in B_m\} $$

The final calibrator $g(\cdot)$ is formed from the normalized one-vs-all calibrators as
$$g( \mathbf{\hat{p_i}}) := (g_1(\hat p_{i,1}), \cdots g_c(\hat p_{i,c})) / \sum_j g_j(\hat p_{i,j})$$

\subsection{Bayesian Binning Quantiles (BBQ)}\label{appendix:bbq}

The BBQ approach of \citet{naeini2015obtaining} extends histogram binning to consider multiple binning schemes at once. Given a histogram calibrator with $B$ total bins, BBQ assigns a score based on the Gamma function ($\Gamma(\cdot)$) of the form:
$$ \operatorname{Score}(M) := \prod_{b=1}^B \frac{\Gamma(\frac{N^\prime}{B})}{\Gamma(N_b + \frac{N^\prime}{B})} \frac{\Gamma(m_b + \alpha_b)}{\Gamma(\alpha_b)} \frac{\Gamma(n_b + \beta_b)}{\Gamma(\beta_b)} $$
where $N_b$ is the total number of samples in bin $b$, $n_b$ is the number of negative class samples in bin $b$, $m_b$ the number of positive class samples and $\alpha_b := \frac{2}{B}p_b, \beta_b := \frac{2}{B}(1-p_b)$ where $p_b$ is the midpoint of the interval defining bin $b$. Given a set of binning calibrators $M_1, \dots M_n$ with associated scores $s_1, \dots, s_n$, the final BBQ calibrator is formed from the weighted average
$M(p) := \frac{\sum_{i} s_i\cdot M_i(p)}{\sum_{i}s_i}$. In the multi-class setting we consider a one-vs-all setting where we train a BBQ calibrator for each class. In our federated setting, and to simplify things under DP, we consider a histogram calibrator with $2^M$ total bins. 
We have clients compute positive and negative histograms over each class for this binning scheme.
Then to utilize BBQ, we merge the $2^M$ bins of  
each histogram in powers of $2$ to create multiple histogram calibrators with total bin counts in the range $\{2, 4, \dots, 2^{M-1}, 2^M\}$ and assign the BBQ scores defined above. %
This defines our FedBBQ protocol. 

\subsection{Weighting Scheme for Binning Methods}
\label{appendix:weights}

In the non-DP setting we apply a classwise weighting scheme with $\alpha_j = \text{clip}(\frac{\tilde N_j}{N_j}, 1)$ where $\tilde N_j$ is the number of aggregated class $j$ examples and $N_j$ is the total number of class $j$ examples. Note that this weighting scheme requires no additional rounds of communication from clients as it is simply a post-processing step the server performs (summing a histogram) on the aggregated histogram received from clients. In the DP setting, knowledge about $N_j$ is private. 
We can still estimate $\tilde N_j$ under DP by summing all the counts from the (noisy) positive histogram for class $j$. 
We replace $N_j$ with the expected error of measuring the histogram under Gaussian noise, $\tilde \alpha_j := \operatorname{clip}(\tilde N_j / \sqrt{2/\pi} \sigma |P_j|, 1)$. In other words, if the signal-to-noise ratio of the histogram is low then we place less weight on the binning prediction.

\subsection{Order-preserving Training for Scaling Methods}
\label{appendix:op}
The method of \cite{rahimi2020intra} allows training a calibrator $g$ to be \emph{order-preserving}.

\begin{definition}[Order-preserving]
    We say a calibrator $g$ is order-preserving for any $\mathbf{p}$ if both $\mathbf{p}$ and $g(\mathbf{p})$ share the same ranking i.e., for all $i, j \in [c]$ we have $\mathbf{p}_i \geq \mathbf{p_j}$ if and only if $g_i(\mathbf{p}) \geq g_j(\mathbf{p})$
\end{definition}
This guarantees that the top-label prediction (i.e., top-1 accuracy) will not be changed after calibration. Rahimi et al. show that a necessary and sufficient condition for $g$ to be order-preserving is if $g$ is of the form 
\begin{equation}\label{eq:op}
    g(\mathbf{p}) := S(\mathbf{p})^{-1}U\mathbf{w}(\mathbf{p})
\end{equation}
 with $U$ being an upper-triangular matrix of ones, $S$ being a sorting permutation and $\mathbf{w}$ a function that satisfies
\begin{itemize}
    \item $\mathbf{w}_i(\mathbf{p}) = 0$ if $\mathbf{y}_i = \mathbf{y}_{i+1}$ and $i < n$
    \item $\mathbf{w}_i(\mathbf{p}) > 0$ if $\mathbf{y}_i > \mathbf{y}_{i+1}$ and $i < n$
\end{itemize}
where $\mathbf{y} := S(\mathbf{p})\mathbf{p}$ is the sorted version of the class-probabilities $\mathbf{p}$ (see Theorem 1 in \cite{rahimi2020intra}). Rahimi et al. further show that for any scaling calibrator $\mathbf{m}(\bold{p})$, this can be achieved by setting $\mathbf{w}(\mathbf{p}) := |\mathbf{y}_i - \mathbf{y}_{i+1}| \mathbf{m}(\mathbf{p})$ and plugging $\mathbf{w}(\mathbf{p})$ into (\ref{eq:op}) to obtain the final order-preserving calibrator $g$.

\end{document}